\documentclass{article}

\usepackage[preprint]{neurips_2024}
\usepackage{tikz}
\usetikzlibrary{matrix}
\usepackage{bbm}
\usepackage{xypic}
\usepackage{stmaryrd}
\usepackage{mathrsfs}
\usepackage{todonotes}
\usepackage{natbib}

\usepackage[utf8]{inputenc} 
\usepackage[T1]{fontenc}    
\usepackage{hyperref}       
\usepackage{url}            
\usepackage{booktabs}       
\usepackage{amsfonts}       
\usepackage{nicefrac}       
\usepackage{microtype}      
\usepackage{xcolor}         
\usepackage{graphicx}
\usepackage{subcaption}
\usepackage{tcolorbox}
\usepackage{amsthm,amsfonts,amssymb,amsmath,bbm}

\theoremstyle{plain}   
\newtheorem{theorem}{Theorem}[section]
\newtheorem{lemma}[theorem]{Lemma}
\newtheorem{definition}[theorem]{Definition}

\theoremstyle{remark}  
\newtheorem{remark}[theorem]{Remark}
\newtheorem{example}[theorem]{Example}

\newcommand{\wstar}{{w^{*}}}
\DeclareMathOperator{\SGLD}{SGLD}

\usepackage[capitalize,noabbrev]{cleveref}
\crefname{equation}{}{} 

\NewDocumentCommand{\task}{O{}}{\mathbf{t}_{#1}}
\NewDocumentCommand{\taskdist}{O{}}{\mathcal{T}_{#1}}
\newcommand{\pilesub}[1]{\texttt{#1}}

\DeclareMathOperator{\Hom}{Hom}

\newtcbox{\tokenbox}{%
  fontupper=\ttfamily,
  colback=gray!10,
  boxrule=0pt,             
  arc=2pt,
  boxsep=0pt,
  frame empty,
  left=2pt,
  right=2pt,
  top=2pt,                 
  bottom=2pt,              
  nobeforeafter,
  valign=center,
  baseline,
  tcbox raise base,
  verbatim,                
  before upper={\vphantom{Äg}},
}

\newtcbox{\tokenboxline}{%
  fontupper=\ttfamily,
  colback=gray!10,
  boxrule=0.5pt,           
  arc=2pt,
  boxsep=0pt,
  left=2pt,
  right=2pt,
  top=2pt,                 
  bottom=2pt,              
  nobeforeafter,
  valign=center,
  baseline,
  tcbox raise base,
  verbatim,
  before upper={\vphantom{Äg}},
}

\title{Modes of Sequence Models and Learning Coefficients}

\author{%
Zhongtian Chen\\
    University of Melbourne\\
\And
  Daniel Murfet\\
    Timaeus\\
}

\begin{document}

\maketitle

\begin{abstract} We develop a geometric account of sequence modelling that links patterns in the data to measurable properties of the loss landscape in transformer networks. First, we cast conditional sequence distributions into a Hilbert‐space framework and apply tensor decompositions to identify their principal \emph{modes}. Truncating the small‐amplitude modes yields an \emph{effective} data distribution that preserves dominant structure while discarding statistical detail. Second, we show theoretically that Local Learning Coefficient (LLC) estimates are insensitive to modes below a data‐dependent threshold. Consequently, the LLC calculated in practice characterises the geometry of the effective rather than the true distribution. This insight clarifies why reliable LLC estimates can be obtained even when a network parameter is not a strict minimiser of the population loss, and it highlights how the inverse temperature in SGLD acts as a resolution dial on the landscape structure. \end{abstract}

\tableofcontents

\section{Introduction}

Statistical models of sequence distributions form the foundation of our mathematical theories of communication \citep{shannon1948mathematical}, artificial general intelligence \citep{hutter2005universal}, and natural language. The empirical success of transformer neural networks has brought sequence models once again to the center of both theoretical and practical aspects of artificial intelligence \citep{vaswani2017attention}. Recent work has suggested that the \emph{geometry} of the loss landscape in these networks may be deeply related to the way that the models compute \citep{chen2023tms1,wang2024differentiationspecializationattentionheads,hoogland2024developmental}, and therefore understanding this geometry and how it relates to structure in the data distribution is a promising path towards interpretability and applications in AI safety \citep{lehalleur2025eataialignment}.

Within the framework of singular learning theory (SLT) \citep{watanabeAlgebraicGeometryStatistical2009} the main way we currently access this geometry empirically is through approximate Bayesian posterior sampling using stochastic gradient Langevin dynamics (SGLD) \citep{wellingBayesianLearningStochastic2011}. This forms the basis for local learning coefficient (LLC) estimation \citep{quantifdegen} and related techniques like susceptibilities \citep{suscep}. However there are many open conceptual and technical questions about this SGLD-based approach to studying the geometry of the loss landscape. 

In this paper, we lay foundations for understanding the role of \emph{modes of the data distribution} in shaping our measurements of the geometry of the loss landscape of sequence models like transformers. We do this by developing an \emph{effective theory} of sequence modelling. Our approach uses function spaces and tensor decompositions to identify the fundamental patterns -- which we call \emph{modes} -- in sequence distributions. This framework allows us to perform a principled coarse-graining of the true data distribution by truncating smaller modes while preserving the dominant structures, resulting in what we term the \emph{effective true distribution}.

Our key insight is that the geometry of learning, as captured by LLC estimation, is inherently limited in its sensitivity to the full spectrum of modes. Due to constraints on the sensitivity of the model's gradients to the long tail of modes, LLC estimation primarily measures the geometry of the effective true distribution rather than the complete distribution. This has several implications for how we interpret measurements of landscape geometry in deep sequence models.

The contributions we make in this paper:
\begin{itemize}
\item \textbf{Modal decomposition of sequence distributions}: Building on existing work in the mathematical literature on natural language processing, we use tensor decompositions to formulate a Hilbert space of sequence models with a natural basis. We show how patterns in the data distribution correspond to modes, and how tensor truncation provides a principled approach to coarse-graining.

\item \textbf{LLC estimation measures effective distributions}: We prove that under reasonable assumptions, LLC estimation is insensitive to small modes in the distribution. Specifically, the LLC estimates we perform in practice (e.g., on transformer neural networks) effectively measure the geometry of a coarse-grained distribution rather than the complete distribution.
\end{itemize}

The paper is organized as follows: In Section \ref{section:background}, we review necessary background on tensor algebra and singular value decomposition. Section \ref{section:sequence_models} introduces our formalism for sequence models and their corresponding data distributions. Section~\ref{section:function_space} develops the Hilbert space framework and modal decomposition. Section~\ref{section:examples} makes these abstract definitions concrete with some examples of modes. Then in Section \ref{section:mode_insensitivity} and Section \ref{section:neg} we connect these concepts to Singular Learning Theory and the Local Learning Coefficient and prove our main results.

\section{Background}\label{section:background}

\subsection{Tensor Algebra}\label{section:tensor_algebra}

Throughout all vector spaces are over $\mathbb{R}$, and all tensor products are of $\mathbb{R}$-vector spaces. Given vector spaces $V$ and $W$, let $\Hom_{\mathbb{R}}(V, W)$ denote the space of all linear transformations from $V$ to $W$. There is a canonical isomorphism of vector spaces
\begin{gather}
\Phi: V^* \otimes W \overset{\cong}{\longrightarrow} \Hom_{\mathbb{R}}(V,W)\,,\\
\Phi(\eta \otimes w)(v) = \eta(v) w\,.
\end{gather}
The \emph{free vector space} on a finite set $S$ is the set
\begin{equation}
\mathbb{R}^S = \{ f: S \longrightarrow \mathbb{R} \}
\end{equation}
with the usual vector space structure of addition of functions and multiplication of a function by a scalar. We define an injective map $S \rightarrow \mathbb{R}^S$ sending $s \in S$ to the delta function $\delta_s(t) = \delta_{s,t}$ and henceforth use this injection to identify elements of $S$ with vectors in $\mathbb{R}^S$.

There is a canonical isomorphism of vector spaces
\begin{gather}
\mathbb{R}^S \otimes \mathbb{R}^T \overset{\cong}{\longrightarrow} \mathbb{R}^{S \times T}\\
s \otimes t \longmapsto (s,t)
\end{gather}
and in particular for $k \ge 1$ we have $(\mathbb{R}^S)^{\otimes k} \cong \mathbb{R}^{S^k}$.

\subsection{Singular Value Decomposition}\label{section:background_svd}

For the reader's convenience we give a minimal treatment of SVD that makes explicit the fact that SVD is an operation on linear transformations between inner product spaces. Let $V, W$ be finite-dimensional inner product spaces over $\mathbb{R}$ with pairings denoted $\langle -, - \rangle_V$ and $\langle -, - \rangle_W$ respectively. Given a linear transformation $A: V \longrightarrow W$ the adjoint $A^\dagger: W \longrightarrow V$ is uniquely defined by
\[
\langle Av, w \rangle_W = \langle v, A^\dagger w \rangle_V \qquad \forall v \in V, w \in W\,.
\]
Then $A^\dagger A, A A^\dagger$ are self-adjoint operators on $V, W$ respectively and thus we may apply the spectral theorem to obtain an orthonormal basis of eigenvectors for these spaces. Let $\{ v_\alpha \}_\alpha$ be an orthonormal basis of $V$ consisting of eigenvectors of $A^\dagger A$ with eigenvalues $\lambda_\alpha \ge 0$. We can divide the indices $\alpha$ into those with positive and zero eigenvalues
\[
\Lambda^{+} = \{ \alpha \mid \lambda_\alpha \neq 0 \}\,, \qquad \Lambda^0 = \{ \alpha \mid \lambda_\alpha = 0 \}\,.
\]
We write $\Lambda = \Lambda^+ \cup \Lambda^0$ for the complete set of indices, so $|\Lambda| = \dim V$. Then we define
\begin{equation}
u_\alpha = \lambda_\alpha^{-1/2} A v_\alpha\,, \qquad s_\alpha = \lambda_\alpha^{1/2}\,, \qquad \alpha \in \Lambda^+\,.
\end{equation}

\begin{lemma} The set $\{ u_\alpha \}_{\alpha \in \Lambda^+}$ is an orthonormal basis in $W$ for the image of $A$. We have
\begin{equation}
A v_\alpha = s_\alpha u_\alpha\,, \quad A^\dagger u_\alpha = s_\alpha v_\alpha\,.
\end{equation}
We call $s_\alpha$ the \emph{singular values}, $v_\alpha$ the \emph{right singular vectors} and $u_\alpha$ the \emph{left singular vectors}.
\end{lemma}
\begin{proof}
If $\lambda_\alpha \neq 0$ define $u'_\alpha = A v_\alpha$. Then
\begin{gather*}
A^\dagger u'_\alpha = A^\dagger A v_\alpha = \lambda_\alpha v_\alpha
\end{gather*}
and hence
\[
AA^\dagger u'_\alpha = \lambda_\alpha A v_\alpha = \lambda_\alpha u'_\alpha
\]
Moreover by adjointness
\[
\langle u'_\alpha, u'_\beta \rangle_W = \langle u'_\alpha, A v_\beta \rangle_W = \langle A^\dagger u'_\alpha, v_\beta \rangle_V = \lambda_\alpha \langle v_\alpha, v_\beta \rangle_V = \lambda_\alpha \delta_{\alpha, \beta}\,.
\]
If we set $u_\alpha = \lambda_\alpha^{-1/2} u'_\alpha = \lambda_\alpha^{-1/2} A v_\alpha$ then the claims follow.
\end{proof}

\subsection{Learning Coefficient}\label{section:learning_coefficient}

In this section we formulate the problem of learning the conditional true distribution $q(y|x)$ of sequences in the context of singular learning theory (SLT) \cite{watanabeAlgebraicGeometryStatistical2009}. We assume $k,l > 0$ and $x \in \Sigma^k$, $y \in \Sigma^l$. Given a set of $n$ independent and identically distributed samples $D_n = \{(x_i, y_i)\}_{i=1}^n$ drawn from the true data-generating distribution $q(x,y)$, we define the sample negative log likelihood $L_n(w)$ associated with a model $p(y|x,w)$ and its theoretical counterpart $L$ as
\begin{align*}
L_n(w) &= -\frac{1}{n} \sum_{i=1}^n \log p(y_i|x_i, w)\,,\\
L(w) &= -\mathbb{E}_{q(x,y)} \big[\log p(y|x, w)\big]\,.
\end{align*}
Under some technical hypotheses the \emph{local learning coefficient} (LLC) \citep{watanabeAlgebraicGeometryStatistical2009} \citep{quantifdegen} can be defined at a local minimum $\wstar$ of $L(w)$ is the \emph{volume scaling exponent} of the loss near $\wstar$ in the following sense. Let $B(\wstar)$ be a closed ball centered on $\wstar$ such that for all $w\in B(\wstar)$, $L(w) \ge L(\wstar)$. For sufficiently small $\epsilon$, consider the set $B(\wstar, \epsilon) = \{w \in B(\wstar) \mid L(w) - L(\wstar) < \epsilon\}$, the volume of which we define as $V(\epsilon):=\int_{B(\wstar,\epsilon)} \varphi(w)\,dw$ for some positive prior $\varphi(w)$. It can be shown that there exist a positive rational number $\lambda(\wstar)$ and a positive integer $m(\wstar)$ such that $V(\epsilon) \propto \epsilon^{\lambda(\wstar)} (-\log \epsilon)^{m(\wstar) -1}$ asymptotically as $\epsilon \to 0$. The LLC is defined to be the rational number $\lambda(\wstar)$.

For regular statistical models, $\lambda(w^*) = d/2$, where $d$ is the dimension of the parameter space. However, in general $\lambda(w^*)$ can be smaller than $d/2$. The LLC plays a crucial role in the asymptotic expansion of the local free energy
\begin{equation}
F_n(B_\gamma(w^*)) = -\log \int_{B(w^*,\gamma)} \exp\{-nL_n(w)\}\varphi(w) dw
\end{equation}
The asymptotic expansion of the local free energy is given by Watanabe's free energy formula \cite[\S 6.3]{watanabe2018}
\begin{equation}
F_n(B_\gamma(w^*)) = nL_n(w^*) + \lambda(w^*) \log n - (m(w^*) - 1) \log \log n + O_P(1)
\end{equation}
This expansion provides insights into the behavior of the posterior distribution and is the basis for one approach to understanding neural network development \citep{chen2023tms1,hoogland2024developmental,wang2024differentiationspecializationattentionheads}.

\subsection{LLC Estimation} \label{section:llc_estimation}

It is technically challenging to derive theoretical values for the LLC even for deep linear neural networks \citep{AOYAGI2024106132}. To this end, \citet{quantifdegen} introduced an estimator of the LLC and we recall its definition below. Consider the posterior distribution in $w$ localized to $\wstar$ by a Gaussian prior centered there with scale parameter $\gamma >0$:
\begin{equation}
p(w|w^*,\beta, \gamma) \propto \exp \left \{-n \beta L_n(w)-\frac{\gamma}{2}||w-w^*||_2^2 \right \}\,.
\label{eq:local_tempered_posterior}
\end{equation}
The larger $\gamma$ is, the more tightly concentrated this distribution is around $\wstar$. Let $\mathbb{E}_{w|\wstar,\beta, \gamma}$ denote expectation over \eqref{eq:local_tempered_posterior}. The \emph{LLC estimator} is defined by
\begin{equation}
    \hat \lambda_n(w^*) = n\beta \big[\mathbb{E}_{w|\wstar,\beta, \gamma} L_n(w)  - L_n(\wstar) \big].
    \label{eq:hatlambda}
\end{equation}
We note that in the original work of \citet{quantifdegen} we use the inverse temperature $\beta^* = 1/\log n$ derived in \citet{watanabeWidelyApplicableBayesian2013}. However empirically it has been found that other inverse temperatures are more convenient because the SGLD sampler is better behaved \citep{hoogland2024developmental,wang2024differentiationspecializationattentionheads}. That this is a valid estimator follows from \citep[Theorem 4]{watanabeWidelyApplicableBayesian2013} which holds for $\beta = \beta_0/\log n$ with $\beta_0 \neq 1$. The effect of this choice of inverse temperature on the estimates is part of the subject of the present work (see Section \ref{section:discussion}).

We approximate the expectation in this estimator with SGLD \cite{wellingBayesianLearningStochastic2011}. The SGLD update for sampling from the local posterior distribution $p(w|w^*, \beta, \gamma)$ is
\begin{align}
\Delta w_t &= \frac{\epsilon_t}{2}\Bigg[\frac{\beta n}{m} \sum_{i=1}^m \nabla_w \log p(y_{l_i}|x_{l_i}, w_t) + \gamma(w^* - w_t)\Bigg] + \eta_t \nonumber \\
&= \frac{\epsilon_t}{2}\Big[-\beta n \nabla_w L_m(w_t) + \gamma(w^* - w_t)\Big] + \eta_t \,, \label{eq: SGLD update}
\end{align}

where $\eta_t \sim \mathcal{N}(0, \epsilon_t)$ and $\{(x_{l_i}, y_{l_i})\}_{i=1}^m$ is a randomly sampled minibatch. Here the \emph{batch-size} $m$ can be significantly smaller than $n$. Let $\{w_1, \dots, w_T\}$ be samples of the tempered posterior at inverse temperature $\beta$ using the SGLD update \eqref{eq: SGLD update}. Then the \emph{SGLD-based LLC estimator} is 
\begin{equation}
    \hat{\lambda}_n^{\SGLD}(\wstar) = n\beta \bigg[ \frac{1}{T} \sum_{t=1}^T L_n(w_t) - L_n(w^*) \bigg]\,. \label{eq: SGLD based LLC estimator}
\end{equation}
Typically we will drop the superscript and write $\hat\lambda_n(\wstar)$ for estimates derived using SGLD samples.

\section{Sequence Models}\label{section:sequence_models}

Let $\Sigma$ be a nonempty finite set which we refer to as the set of \emph{symbols} or \emph{tokens}. Elements of $\Sigma$ are denoted by Greek letters $\sigma, \tau, \ldots$. For $k \in \mathbb{Z}_{\geq 1}$, let $\Sigma^k$ denote the set of sequences of symbols in $\Sigma$, which we refer to as \emph{strings}, with length $k$. General strings are denoted by $x,y,z,\ldots$. We denote by $K$ the maximum length of a string.

\subsection{Data Distribution}

We consider jointly distributed random variables $X_1,\ldots,X_K$ taking values in $\Sigma \subseteq \mathbb{R}^\Sigma$. Here elements in $\Sigma$ are identified as vectors in $\mathbb{R}^\Sigma$ (see Section \ref{section:tensor_algebra}). For each $k$ with $1 \leq k \leq K$, the strings of length $k$ are distributed according to a joint distribution $q_k(X) = q_k(X_1, \ldots, X_k)$ where we use $X_i$ to stand for the $i$th position in a string over any length.  We assume $q_1(X = x) > 0$ for all $x \in \Sigma$. The cross moments of the random vectors $X_1,\ldots,X_k$ are joint probabilities over tokens
\begin{align*}
\mathbb{E}_{q_k}[ X_1 \otimes \cdots \otimes X_k ] &= \sum_{x \in \Sigma^k} q_k( X_{1} = x_1, \ldots, X_{k} = x_k ) x_1 \otimes \cdots \otimes x_k\,.
\end{align*}
For any $X \in \mathbb{R}^\Sigma$, let $X_{\sigma}$ denote the number at the $\sigma$-th coordinate of $X$. Note that 
\begin{align*}
    \mathbb{E}_{q_k}[ (X_1)_{\sigma_1} \cdots (X_k)_{\sigma_k}] &= \sum_{x \in \Sigma^k}q_k\big((X_1)_{\sigma_1} = x_1, \dots, (X_k)_{\sigma_k} = x_k\big) \delta_{\sigma_1,x_1} \cdots \delta_{\sigma_k,x_k}\\
    &= q_k( X_{1} = x_1, \ldots, X_{k} = x_k )
\end{align*}
so the cross moments can also be written as 
\[
    \mathbb{E}_{q_k}[ X_1 \otimes \cdots \otimes X_k ] = \sum_{x \in \Sigma^k} \mathbb{E}_{q_k}[ (X_1)_{x_1} \cdots (X_k)_{x_k}] x_1 \otimes \cdots \otimes x_k\,.
\]
\begin{definition}\label{definition:fund_tensor} We call $A_k = \mathbb{E}_q[ X_1 \otimes \cdots \otimes X_k ] \in (\mathbb{R}^\Sigma)^{\otimes k}$ the $k$th \emph{fundamental tensor}.
\end{definition} 

We will use the linear map
\begin{align}
    \varepsilon: \mathbb{R}^\Sigma \to \mathbb{R}, \quad &\varepsilon(\sigma) = 1\,.
\end{align}
Observe that applying $\varepsilon \otimes \mathbbm{1}^{\otimes k - 1}$ to $A_k$ has the effect of marginalizing over $X_1$, as
\begin{align*}
\big( \varepsilon \otimes \mathbbm{1}^{\otimes k - 1} \big)(A_k) &= \sum_{x \in \Sigma^k} \mathbb{E}_{q_k}[ (X_1)_{x_1} \cdots (X_k)_{x_k}] \big( \varepsilon \otimes \mathbbm{1}^{\otimes k - 1} \big) \big( x_1 \otimes \cdots \otimes x_k \big)\\
&= \sum_{x \in \Sigma^k} \mathbb{E}_{q_k}[ (X_1)_{x_1} \cdots (X_k)_{x_k}] x_2 \otimes \cdots \otimes x_k\\
&= \sum_{x_2, \ldots, x_k \in \Sigma} \mathbb{E}_{q_k}[ (X_2)_{x_2} \cdots (X_k)_{x_k}] x_2 \otimes \cdots \otimes x_k\,.
\end{align*}
More generally, applying $\varepsilon$ to $l$ positions in a fundamental tensor $A_k$ has the effect of producing a tensor in $(\mathbb{R}^\Sigma)^{\otimes k - l}$ which encodes the appropriate marginal distribution of $q_k$. We refer to the basis of tensors $\sigma_1 \otimes \cdots \otimes \sigma_k$ for $(\mathbb{R}^{\Sigma})^{\otimes k}$ as the \emph{word basis}. Since the fundamental tensors encode probability distributions,  their coefficients in the word basis are non-negative, and $\varepsilon^{\otimes k} A_k = 1$.

In general the collection of distributions $\{q_k\}_{k=1}^K$ need not have any particular relation to one another, but if we are to imagine these distributions as being the ``true'' distribution of natural language then we may impose a natural condition which says that $q_k$ is equal to the marginal distribution of \emph{sequences of length $k$ within sequences of length $K$}. 

The following definition is a reformulation of \citep[Definition 1.1]{mumford1994pattern}:

\begin{definition}\label{defn:language} We say that the sequence $\{q_k\}_{k=1}^K$ is a \emph{language} if the following diagram commutes for any $i + j = K - k$
\[
\xymatrix@C+2pc{
\mathbb{R} \ar[r]^{A_K}\ar[dr]_-{A_k} & (\mathbb{R}^\Sigma)^{\otimes K} \ar[d]^{\varepsilon^{\otimes i} \otimes \mathbbm{1}^{\otimes k} \otimes \varepsilon^{\otimes j}}\\
& (\mathbb{R}^\Sigma)^{\otimes k}
}
\]
\end{definition}

It follows in turn that marginalizations of $q_k$ yield $q_l$ for $l < k$. In the rest of the paper we assume we are given a language in this sense, and as a consequence we henceforth drop the subscripts and write $q$ for the distribution over sequences of tokens. We also adopt an abbreviated notation for the probability of a sequence of tokens
\begin{equation}\label{eq:abbrev_q}
q(\sigma_1 \cdots \sigma_k) := q(X_1 = \sigma_1, \ldots, X_k = \sigma_k)\,.
\end{equation}
In this notation
\begin{equation}
A_k = \sum_{\sigma \in \Sigma^k} q(\sigma_1 \cdots \sigma_k) \, \sigma_1 \otimes \cdots \otimes \sigma_k\,.
\end{equation}

Similarly we write
\begin{equation}\label{eq:abbrev_cond_q}
q\big( \tau_1 \cdots \tau_l \mid \sigma_1 \cdots \sigma_k \big) := q\big( X_{k+1} = \tau_1, \ldots, X_{k+l} = \tau_l \mid X_1 = \sigma_1, \ldots, X_k = \sigma_k \big)
\end{equation}

\subsection{Tokenisation}

Tokenisation is a fundamental preprocessing step in modern language model training, serving as the interface between raw text and the numerical representations required by neural networks. A tokeniser segments input text into discrete units called tokens, which form the vocabulary that the model operates on \citep{kudo2018sentencepiece, sennrich2016neural,phuong2022formal}. The choice of tokenisation strategy significantly impacts model performance, training efficiency, and downstream applicability \citep{mielke2021between}.

Several tokenisation approaches have evolved alongside advances in language modelling:

\paragraph{Word-level tokenisation.} Early approaches split text on whitespace and punctuation, treating each word as a distinct token \citep{mikolov2013efficient}. While intuitive, this approach suffers from vocabulary explosion in morphologically rich languages and an inability to handle out-of-vocabulary words.

\paragraph{Character-level tokenisation.} Character-based approaches \citep{kim2016character} use individual characters as tokens, ensuring a small vocabulary size and eliminating out-of-vocabulary issues. However, they create very long sequences and lose direct access to word-level semantics, requiring the model to learn these relationships from scratch.

\paragraph{Subword tokenisation.} Modern language models predominantly use subword tokenisation algorithms \citep{sennrich2016neural, kudo2018sentencepiece}, which strike a balance between word and character approaches. These methods adaptively split words into subword units based on frequency statistics in the training corpus.

The choice of tokenisation strategy creates implicit biases in how models learn linguistic patterns \citep{rust2021good}. These biases manifest in several ways. For example subword algorithms like BPE create shorter tokens for frequent words or substrings and longer tokens for rare ones. The distributional properties we study, including modes, are shaped by how the continuous stream of text is discretised through tokenisation.

\begin{example} In the tokeniser for \texttt{EleutherAI/pythia-70m} \cite{biderman2023pythia} which we use for our experiments a sample taken from the Pile subset \pilesub{wikipedia\_en} \cite{gao2020pile,devingulliver2025monology} is tokenised as follows:
\begin{center}
\tokenbox{~In} \tokenbox{~sh} \tokenbox{i} \tokenbox{it} \tokenbox{ake} \tokenbox{~m} \tokenbox{ush} \tokenbox{ro} \tokenbox{oms} \tokenbox{~g} \tokenbox{rown} \tokenbox{~out} \tokenbox{do} \tokenbox{ors} \tokenbox{~on} \tokenbox{wood} \tokenbox{~in} \tokenbox{~} \tokenbox{a} \tokenbox{~city} \tokenbox{~in} \tokenbox{~} \tokenbox{the} \tokenbox{~pre} \tokenbox{fect} \tokenbox{ure} \tokenbox{~I} \tokenbox{b} \tokenbox{ar} \tokenbox{ak} \tokenbox{i}\,.
\end{center}
In this case the set $\Sigma$ of tokens has around 50k elements. This tokenizer was developed in \cite{black2022gpt} and is a BPE tokenizer that is trained specifically on the Pile.
\end{example}

\subsection{Transformers}\label{section:models_transformer}

Transformer neural networks \citep{vaswani2017attention} have become the dominant architecture for modelling sequential data, particularly in natural language processing. For a formal definition of a transformer-based sequence model see \cite{phuong2022formal}.

For autoregressive modelling, transformers typically employ causal attention masks to ensure that the prediction for position $t$ only depends on tokens at positions $< t$. Transformers are trained to \emph{simultaneously} model the conditional distributions $q(X_{k+1} | X_1, \ldots, X_{k})$ for all sequence lengths $1 \le k < K$ and thus implicitly to model the joint distribution
\[
q(X_1 \cdots X_K) = \prod_{k=1}^{K-1} q(X_{k+1} | X_{\le k}) q(X_1)\,.
\]
We assume $q(X_1)$ is given and does not need to be modelled and $X_{\le k} = X_1 \cdots X_k$. Here $K$ is called the \emph{maximum context length}. The training loss for such language models is usually defined for a sample $\{ S^i_K \}_{i=1}^n$ of \emph{contexts} $S^i_K \in \Sigma^K$ which are drawn i.i.d from a large corpus, by
\[
\ell_n(w) = - \frac{1}{n} \sum_{i=1}^n \sum_{k=1}^{K-1} \log p( x^i_{k+1} | x^i_{\le k}, w)
\]
where we write $S^i_K = x^i_1 x^i_2 \cdots x^i_K$. Here for any sequence $x \in \Sigma^k$ and $y \in \Sigma$
\[
p(y|x,w) := \operatorname{softmax}\big( g_w( x ) \big)[y]
\]
where $p[y]$ denotes the $y$ entry in the softmax vector \citep{phuong2022formal} and $g_w$ is the transformer parametrised by weights $w$. Thus the transformer is trained to model the distribution over sequences of length $K$ \emph{given} the distribution over unigrams $q_1(x)$ where we define
\[
p(X_1 \cdots X_K|w) := \prod_{k=1}^{K-1} p(X_{k+1} | X_{\le k}, w) q_1(X_1)\,.
\]
Note that then by definition
\begin{equation}\label{eq:training_loss_transformer_normal}
\ell_n(w) = - \frac{1}{n} \sum_{i=1}^n \log \prod_{k=1}^{K-1} p(x^i_{k+1}|x^i_{\le k}, w) = -\frac{1}{n} \sum_{i=1}^n \log p(S^i_K|w) + \log q_1(x^i_1)\,.
\end{equation}
Of course it is not assumed that $q_1$ is known, but since it does not depend on $w$ it does not appear in gradients and so we may ignore it.

Although it is usually stated that transformers do \emph{next-token} prediction, from a mathematical perspective they just model the distribution of contexts. In particular, given integers $k, l > 0$ with $k + l = K$ it is just as accurate to describe the training loss of a transformer as
\begin{equation}\label{eq:training_loss_transformer}
\ell_n(w) = -\frac{1}{n} \sum_{i=1}^n\Big[ \log p(x^i_{> k}|x^i_{\le k}, w) + \log p(x^i_{\le k}|w) \Big] + \log q_1(x^i_1)
\end{equation}
where $x^i_{>k} = x^i_{k+1} \cdots x^i_K \in \Sigma^l$. That is, we can think of optimising this loss as training the transformer to predict sequences of length $l$ given sequences of length $k$ for any pair with $k + l = K$. 

This observation is not an idle one: since the forward pass of the transformer on a token in position $i$ produces a latent vector which is \emph{used for prediction} in all sequence positions $> i$ it is not surprising that these latent vectors contain information relevant for such prediction \cite{pal-etal-2023-future,wu2024do}. Indeed \citet{shai2024transformers} have argued this is optimal based on computational mechanics. This explains why in this paper we consider models $p(y|x,w)$ which predict $y \in \Sigma^l$ given $x \in \Sigma^k$ (rather than fixing $l = 1$) since from a mathematical perspective this is arguably the correct setting for studying transformers as language models.

\section{Function Space}\label{section:function_space}

\subsection{Hilbert space}\label{section:Hilbert_space}

\begin{definition} For $l > 0$ we make $\mathbb{R}^{\Sigma^l}$ into Hilbert space where for $y,y' \in \mathbb{R}^{\Sigma^l}$
\begin{equation}\label{eq:token_dot_product}
\langle y, y' \rangle_{\mathbb{R}^{\Sigma^l}} = \sum_{\sigma \in \Sigma^l} y(\sigma) y'(\sigma)\,.
\end{equation}
\end{definition}

Here we use that $y,y'$ are functions $\Sigma^l \rightarrow \mathbb{R}$ which we may evaluate on symbols. Recall that we assume $q(x) > 0$ for all $1 \le k \le K$ and $x \in \Sigma^k$.

\begin{definition}\label{definition:Vinner}
For $k > 0$ let $\mathscr{V}_k$ denote the vector space $\mathbb{R}^{\Sigma^k}$ with the inner product
\begin{equation}
    \langle u, v \rangle_{\mathscr{V}_k} = \sum_{x \in \Sigma^{k}} u(x)v(x)q(x)^{-1}.
\end{equation}
When $k = 1$ we sometimes write $\mathscr{V}$ for $\mathscr{V}_1$.
\end{definition}

\begin{remark}\label{remark:normalised_symbol} Given $x \in \Sigma^k$ we have
\[
\Vert x \Vert_{\mathscr{V}_k}^2 = q(x)^{-1}
\]
hence the unit vector in the direction of $x$ is $q(x)^{1/2} x$. Hence the vectors $\{ q(x)^{1/2} x \}_{x \in \Sigma^k}$ form an orthonormal basis for $\mathscr{V}_k$.
\end{remark}

Given a measure space $(X, \mu)$ and Hilbert space $\mathscr{K}$ there is a Hilbert space of square-integrable $\mathscr{K}$-valued functions $L^2(X,\mu ; \mathscr{K})$ \citep[II.1]{reed1980methods}. In this paper we make $\Sigma^k$ into a measure space using $q_k$ and consider the Hilbert space 
\[
\mathscr{H}_{k,l} = L^2(\Sigma^k, q_k;\mathbb{R}^{\Sigma^l})
\]
with pairing for functions $f,g: \Sigma^k \longrightarrow \mathbb{R}^{\Sigma^l}$ in $\mathscr{H}_{k,l}$ defined by
\begin{gather}
\langle f, g \rangle_{\mathscr{H}_{k,l}} = \sum_{x \in \Sigma^k} \big\langle f(x), g(x) \big\rangle_{\mathbb{R}^{\Sigma^l}} \, q(x)\\
\Vert f \Vert_{\mathscr{H}_{k,l}} = \Big( \sum_{x \in \Sigma^k} \Vert f(x) \Vert_{\mathbb{R}^{\Sigma^l}}^2 \, q(x) \Big)^{1/2}\,.
\end{gather}
In the case $k = 0$, we take $\mathscr{H}_0 = \mathbb{R}^{\Sigma^l}$ as a Hilbert space with the pairing in \eqref{eq:token_dot_product}. 

Recall that for a finite set $S$ we write
\[
\Delta S = \{ g \in \mathbb{R}^S \, | \, \sum_{s \in S} g(s) = 1 \text{ and } g(s) \in [0,1] \text{ for all } s \in S \}\,.
\]

\begin{definition}\label{defn:pkl} Let $\mathscr{P}_{k,l} \subseteq \mathscr{H}_{k,l}$ denote the subset of functions $f: \Sigma^k \longrightarrow \mathbb{R}^{\Sigma^l}$ satisfying
\[
f(x) \in \Delta\big( \Sigma^l \big) \,, \quad \forall x \in \Sigma^k\,.
\]
This is the subset of conditional probability distributions.
\end{definition}

We note that $\mathscr{P}_{k,l}$ is a product of probability simplices, and thus a smooth manifold with corners embedded into the affine space $\mathscr{H}_{k,l}$ viewed as a real manifold. In fact $\mathscr{P}_{k,l}$ is a statistical manifold in the sense of \cite{amari2016information}.

\subsection{Converting Between Tensors and Linear Transformations}

Since all functions on a finite set are integrable, the underlying set of $\mathscr{H}_{k,l}$ is just the set of all functions $\Sigma^k \longrightarrow \mathbb{R}^{\Sigma^l}$ which is in bijection with the vector space of linear transformations $(\mathbb{R}^\Sigma)^{\otimes k} \longrightarrow (\mathbb{R}^\Sigma)^{\otimes l}$. This bijection identifies a function $f: \Sigma^k \rightarrow \mathbb{R}^{\Sigma^l}$ with the linear transformation
\begin{gather}
(\mathbb{R}^\Sigma)^{\otimes k} \longrightarrow (\mathbb{R}^\Sigma)^{\otimes l}\,,\\
\sigma_1 \otimes \cdots \otimes \sigma_k \longmapsto f(\sigma_1 \cdots \sigma_k)\,.
\end{gather}
Note that we used the fact that $(\mathbb{R}^{\Sigma})^{\otimes l}$ and $\mathbb{R}^{\Sigma^l}$ are identified with each other via 
\begin{gather*}
    (\mathbb{R}^{\Sigma})^{\otimes l} \overset{\cong}{\longrightarrow} \mathbb{R}^{\Sigma^l} \, , \\
    \sigma_1 \otimes \cdots \otimes \sigma_l \longmapsto (\sigma_1, \dots, \sigma_l) \,.
\end{gather*}
There is an isomorphism of vector spaces
\begin{gather}
\mathscr{H}_k \overset{\cong}{\longrightarrow} \big[ \mathbb{R}^{\Sigma^k} \big]^* \otimes \mathbb{R}^{\Sigma^l}\,, \label{eq:curry}\\
f \longmapsto \sum_{x \in \Sigma^k} x^* \otimes f(x)\,. \nonumber
\end{gather}
Here if $x = \sigma_1 \cdots \sigma_k$ we write $x^*$ for the functional that sends a tensor in $(\mathbb{R}^\Sigma)^{\otimes k}$ to its coefficient in $\sigma_1 \otimes \cdots \otimes \sigma_k$, using the expansion in word basis. While the definition of dual vector $x^*$ makes use of the word basis, the isomorphism \eqref{eq:curry} nonetheless doesn't depend on a choice of basis (because $x$ appears again in $f(x)$).

There is no canonical isomorphism between a vector space and its dual. A nondegenerate bilinear form provides such an identification. Using the bilinear forms $\langle -, - \rangle_{\mathbb{R}^{\Sigma^l}}$ and $\langle -, - \rangle_{\mathscr{V}_k}$ we have the isomorphisms of vector spaces
\begin{align}
\mathbb{R}^{\Sigma^k} &\overset{\cong}{\longrightarrow} (\mathbb{R}^{\Sigma^k})^*\,, \qquad u \longmapsto u^* := \langle u, - \rangle_{\mathbb{R}^{\Sigma^k}}\,,\label{eq:curry5}\\
\mathbb{R}^{\Sigma^k} &\overset{\cong}{\longrightarrow} (\mathbb{R}^{\Sigma^k})^*\,, \qquad u \longmapsto \hat{u}^* := \langle u, - \rangle_{\mathscr{V}_k}\,. \label{eq:curry2}
\end{align}
Note that for $u,v \in \Sigma^k$ we have
\begin{align*}
\hat{u}^*(v) = \langle u, v \rangle_{\mathscr{V}_k} = \delta_{u,v} q(u)^{-1}
\end{align*}
hence $\hat{u}^* = q(u)^{-1} u^*$ as linear functionals.

\begin{lemma}\label{lemma:pairing_final} For $u, v \in \mathbb{R}^{\Sigma^k}$
\[
\langle u, v \rangle_{\mathscr{V}_k} = \sum_{x \in \Sigma^k} q(x) \hat{u}^*(x) \hat{v}^*(x)\,.
\]
\end{lemma}
\begin{proof}
The vectors $\{ q(x)^{1/2} x \}_{x \in \Sigma^k}$ form an orthonormal basis for $\mathscr{V}_k$ by Remark \ref{remark:normalised_symbol}. Hence
\begin{align*}
\langle u, v \rangle_{\mathscr{V}_k} &= \sum_{x \in \Sigma^k} \big\langle u, q(x)^{1/2}x \big\rangle_{\mathscr{V}_k} \big\langle q(x)^{1/2}x, v \big\rangle_{\mathscr{V}_k}\\
&= \sum_{x \in \Sigma^k} q(x) \langle u, x \rangle_{\mathscr{V}_k} \langle x, v \rangle_{\mathscr{V}_k}\\
&= \sum_{x \in \Sigma^k} q(x) \hat{u}^*(x) \hat{v}^*(x)
\end{align*}
as claimed.
\end{proof}

\begin{lemma}\label{lemma:inverse_curry2} The inverse of \eqref{eq:curry2} sends a functional $\eta$ to $\sum_{x \in \Sigma^k} q(x) \eta(x) x$.
\end{lemma}
\begin{proof}
It suffices to show $\eta = \sum_x q(x) \eta(x) \hat{x}^*$. We check it on basis elements $u \in \Sigma^k$
\begin{align*}
    \sum_x q(x) \eta(x) \hat{x}^*(u) = \sum_x q(x) \eta(x) q(x)^{-1} x^*(u) = \eta(u)
\end{align*}
as required.
\end{proof}

\begin{lemma}\label{lemma:h_k_iso_detail} The composite isomorphism defined using $\langle - , - \rangle_{\mathscr{V}_k}$
\begin{gather}
\mathscr{H}_k \overset{\eqref{eq:curry}}{\cong} \big[ \mathbb{R}^{\Sigma^k} \big]^* \otimes \mathbb{R}^{\Sigma^l} \overset{\eqref{eq:curry2}}{\cong} \mathbb{R}^{\Sigma^k} \otimes \mathbb{R}^{\Sigma^l} \cong (\mathbb{R}^\Sigma)^{\otimes (k+l)} \label{eq:curry4}
\end{gather}
is defined on $f \in \mathscr{H}_{k,l}$ by
\[
f \longmapsto \sum_{x \in \Sigma^k} q(x)\, x \otimes f(x)\,.
\]
\end{lemma}
\begin{proof}
Under \eqref{eq:curry} a function $f$ is sent to $\sum_{u \in \Sigma^k} u^* \otimes f(u)$, and under \eqref{eq:curry2}, $u^*$ is sent to $q(u) u$ by Lemma \ref{lemma:inverse_curry2}, which proves the claim.
\end{proof}

\begin{lemma}\label{lemma:cond_tensor_adjoint}
The fundamental tensor $A_{k+l} \in (\mathbb{R}^\Sigma)^{\otimes (k+l)}$ corresponds under \eqref{eq:curry4} to $\mathcal{C}_{k,l} \in \mathscr{H}_{k,l}$ defined by
\begin{equation}\label{eq:mathcalC}
\mathcal{C}_{k,l}(x) = \sum_{y \in \Sigma^l} q(y|x) y\,, \qquad \forall x \in \Sigma^k\,.
\end{equation}
\end{lemma}
\begin{proof}
By Lemma \ref{lemma:h_k_iso_detail} the function $\mathcal{C}_{k,l}$ corresponds to the tensor
\begin{align*}
\sum_{x \in \Sigma^k} q(x)\, x \otimes \mathcal{C}_{k,l}(x) &= \sum_{x \in \Sigma^k} q(x) \, x \otimes \sum_{y \in \Sigma^l} q(y|x) y\\
&= \sum_{x \in \Sigma^k, y \in \Sigma^l} q(xy) x \otimes y\\
&= A_{k+l}
\end{align*}
using $q(xy) = q(y|x) q(x)$.
\end{proof}

\subsection{Tensor Decomposition}\label{section:tensor_decomp}

Let $k, l > 0$. To apply dimensional reduction we start with the fundamental tensor $A_{k+l}$ and take the associated linear transformation (Lemma \ref{lemma:cond_tensor_adjoint}):
\begin{gather*}
\mathcal{C}_{k,l}: \mathbb{R}^{\Sigma^k} \longrightarrow \mathbb{R}^{\Sigma^l}\,,\\
x \longmapsto \sum_{y \in \Sigma^l} q(y|x) y\,.
\end{gather*}
It is natural to apply singular value decomposition (SVD), however, this construction is defined for linear transformations between inner product spaces (Section \ref{section:background_svd}) and so we must make a choice of inner product for the domain and codomain. The right and left singular vectors will have different interpretations. The left singular vectors are viewed as elements of the \emph{output} space of the function space $\mathscr{H}_{k,l} = L^2(\Sigma^k,q_k;\mathbb{R}^{\Sigma^l})$ while the right singular vectors are vectors in $\mathbb{R}^{\Sigma^k}$. 

We take the domain to be the inner product space $\mathscr{V}_k$ of Definition \ref{definition:Vinner} and the codomain to be the inner product space $\mathbb{R}^{\Sigma^l}$ with the dot product \eqref{eq:token_dot_product}. That is, we consider $\mathcal{C}_{k,l}$ and its adjoint
\[
\mathcal{C}_{k,l}: \mathscr{V}_k \longrightarrow \mathbb{R}^{\Sigma^l}\,, \quad \mathcal{C}_{k,l}^\dagger: \mathbb{R}^{\Sigma^l} \longrightarrow \mathscr{V}_k
\]
defined uniquely such that for all pairs $x \in \Sigma^k, \sigma \in \Sigma^l$
\begin{equation}\label{eq:adjoint_cal}
\langle \mathcal{C}_{k,l} x, \sigma \rangle_{\mathbb{R}^{\Sigma^l}} = \langle x, \mathcal{C}_{k,l}^\dagger \sigma \rangle_{\mathscr{V}_k}\,.
\end{equation}
The singular value decomposition of $\mathcal{C}_{k,l}$ with respect to the given inner products on the domain and codomain gives (using the notation of Section \ref{section:background_svd}):
\begin{itemize}
    \item an index set $\Lambda_{k,l}$ presented as a disjoint union of subsets $\Lambda^+_{k,l}, \Lambda^0_{k,l}$,
    \item an orthonormal basis of vectors $\{ v_\alpha \}_{\alpha \in \Lambda_{k,l}}$ of $\mathscr{V}_k$, the \emph{right singular vectors} (orthonormal with respect to the pairing of of $\mathscr{V}_k$),
    \item an orthonormal basis of vectors $\{ u_\alpha \}_{\alpha \in \Lambda^+_{k,l}}$ for the image of $\mathcal{C}_{k,l}$ in $\mathbb{R}^{\Sigma^l}$ called the \emph{left singular vectors} (orthonormal with respect to the dot product)
    \item \emph{singular values} $s_\alpha \ge 0$ with $s_\alpha > 0$ for $\alpha \in \Lambda^+_{k,l}$ and $s_\alpha = 0$ for $\alpha \in \Lambda^0_{k,l}$.
\end{itemize}
These have the property that
\begin{align*}
\mathcal{C}_{k,l} \, v_\alpha &= 0\,, && \alpha \in \Lambda_{k,l}^0\,\\
\mathcal{C}_{k,l} \,v_\alpha &= s_\alpha u_\alpha\,, && \alpha \in \Lambda_{k,l}^+\\
\mathcal{C}_{k,l}^\dagger \, u_\alpha &= s_\alpha v_\alpha && \alpha \in \Lambda_{k,l}^+\,.
\end{align*}
By construction
\begin{equation}\label{eq:mode_decomp_A2}
\mathcal{C}_{k,l} = \sum_{\alpha \in \Lambda_{k,l}} s_\alpha u_\alpha \circ \hat{v}_\alpha^* = \sum_{\alpha \in \Lambda^+_{k,l}} s_\alpha u_\alpha \circ \hat{v}_\alpha^*
\end{equation}
where we interpret $u_\alpha: \mathbb{R} \rightarrow \mathbb{R}^{\Sigma^l}$ as the linear transformation sending $1$ to $u_\alpha$, and $\hat{v}_\alpha^*(v_\beta) = \delta_{\alpha, \beta}$. We refer to $\Lambda_{k,l}$ as the set of \emph{modes} of $\mathcal{C}_{k,l}$.

\begin{definition}\label{defn:propensity}
Let $\alpha \in \Lambda^+_{k,l}$. For $x \in \Sigma^k, y \in \Sigma^l$ we define
\begin{equation}\label{eq:propensity_defn}
q(y|x, \alpha) := s^{-1}_\alpha \hat{v}_\alpha^*(x)u_\alpha^*(y) \in \mathbb{R}
\end{equation}
which we call the \emph{propensity} of $y$ to follow $x$ according to the mode $\alpha$. We adopt the convention that the propensity is zero if $\alpha \in \Lambda^0_{k,l}$. The \emph{propensity} of a mode $\alpha \in \Lambda_{k,l}$ is defined to be $q(\alpha) = s_\alpha^2$.
\end{definition}

The preceding discussion is summarised by the \emph{mode decomposition} of the conditional probability:

\begin{lemma}\label{lemma:mode_decomp_cond_prob}
For any $x \in \Sigma^k$ and $y \in \Sigma^l$,
\[
q(y|x) = \sum_{\alpha \in \Lambda_{k,l}} q(y \mid x, \alpha) q(\alpha)\,.
\]
\end{lemma}
\begin{proof}
It follows from \eqref{eq:mode_decomp_A2} that
\begin{equation}\label{eq:decomp_lambda_bigram}
q(y|x) = y^*( \mathcal{C}_{k,l}(x) ) = \sum_{\alpha \in \Lambda_{k,l}} s_\alpha \hat{v}_\alpha^*(x)u_\alpha^*(y)
\end{equation}
from which the claim follows.
\end{proof}

The choice of the $s^{-1}_\alpha$ factor in \eqref{eq:propensity_defn} is justified by Remark \ref{remark:absolute_bigram_propensities}.

Note that while these propensities sum to a probability, they do not have a clear interpretation as likelihoods outside of this sum over $\alpha$. This will not cause confusion, because in this paper there are no probabilities that appear conditioned on mode indices, only propensities.

\begin{remark}\label{remark:compute_dagger} If we define $D: \mathbb{R}^{\Sigma^k} \rightarrow \mathbb{R}^{\Sigma^k}$ by $D(x) = q(x)^{-1/2} x$ then
\begin{equation}\label{eq:V_vs_rsigma_pair}
\langle x,x' \rangle_{\mathscr{V}_k} = \langle Dx, Dx' \rangle_{\mathbb{R}^{\Sigma^k}}\,.
\end{equation}
Let $\mathcal{C}_{k,l}^*$ denote the adjoint with respect to the dot product on both the domain and codomain, so its matrix is the transpose of the matrix of $\mathcal{C}_{k,l}$ in the token basis. Then using \eqref{eq:adjoint_cal}, \eqref{eq:V_vs_rsigma_pair} we have
\[
\langle \mathcal{C}_{k,l} D^{-1} x, x' \rangle_{\mathbb{R}^{\Sigma^k}} = \langle D^{-1} x, \mathcal{C}_{k,l}^\dagger x' \rangle_{\mathscr{V}_k} = \langle x, D\mathcal{C}_{k,l}^\dagger x' \rangle_{\mathbb{R}^{\Sigma^k}}\,.
\]
Hence $D \mathcal{C}_{k,l}^\dagger = (\mathcal{C}_{k,l} D^{-1})^* = D^{-1} \mathcal{C}_{k,l}^*$, so $\mathcal{C}_{k,l}^\dagger = D^{-2} \mathcal{C}_{k,l}^*$. The singular value decomposition of $\mathcal{C}_{k,l}$, viewed as a linear transformation between inner product spaces $\mathscr{V}_k$ and $\mathbb{R}^{\Sigma^l}$, is thus the same as the singular value decomposition of $\mathcal{C}_{k,l} D^{-1}$ viewed as a linear map from $\mathbb{R}^{\Sigma^k}$ and $\mathbb{R}^{\Sigma^l}$.
\end{remark}

\begin{remark}
The matrix of $\mathcal{C}_{k,l}$ is a (column) stochastic matrix since $\sum_{y \in \Sigma^l} q(y|x) = 1$. When $k = l$ the Perron-Frobenius theorem applies to this matrix (noting that we have assumed all probabilities positive) and we can think of this as a Markov chain \cite{seabrook2023tutorial}.
\end{remark}

\subsection{Mode Basis}\label{section:mode_decomp_bigrams}

We make use of the notation established in Section \ref{section:tensor_decomp}. Let $k,l > 0$ be given and let $\Lambda = \Lambda_{k,l}$, $\Lambda^+ = \Lambda_{k,l}^+$ as above. Note that the $u_\alpha$ for $\alpha \in \Lambda^+$ may span a proper subspace of $\mathbb{R}^{\Sigma^l}$. Choose arbitrarily an extension of this to an orthonormal basis $\{ u_\alpha \}_{\alpha \in \Lambda^{++}}$ for all of $\mathbb{R}^{\Sigma^l}$ where $\Lambda^{++} \supseteq \Lambda^+$.

\begin{definition} Given $\alpha \in \Lambda$, $\beta \in \Lambda^{++}$ we define $e_{\alpha\beta} \in \mathscr{H}_{k,l}$ to be
\begin{equation}
e_{\alpha\beta} = u_\beta \circ \hat{v}_\alpha^*
\end{equation}
that is, $e_{\alpha\beta}(x) = \hat{v}_\alpha^*(x) u_\beta$ and so as functions $\Sigma^l \rightarrow \mathbb{R}$
\begin{equation}
e_{\alpha\beta}(x)(y) = \hat{v}_\alpha^*(x) u_\beta^*(y)\,.
\end{equation}
\end{definition}

Hence by \eqref{eq:mode_decomp_A2}
\begin{equation}\label{eq:decomp_ckl}
\mathcal{C}_{k,l} = \sum_{\alpha \in \Lambda^+_{k,l}} s_\alpha e_{\alpha \alpha}\,.
\end{equation}

\begin{lemma}\label{lemma:orthonormal_basis_H1} The vectors $\{ e_{\alpha\beta} \}_{\alpha \in \Lambda, \beta \in \Lambda^{++}}$ form an orthonormal basis of $\mathscr{H}_{k,l}$.    
\end{lemma}
\begin{proof}
We compute that
\begin{align*}
\langle e_{\alpha\beta}, e_{\gamma\delta} \rangle_{\mathscr{H}_{k,l}} &= \sum_{x \in \Sigma^k} q(x) \big\langle e_{\alpha\beta}(x), e_{\gamma\delta}(x) \big\rangle_{\mathbb{R}^{\Sigma^l}}\\
&= \sum_{x \in \Sigma^k} q(x) \hat{v}_\alpha^*(x) \hat{v}_\gamma^*(x) \big\langle u_\beta, u_\delta \big\rangle_{\mathbb{R}^{\Sigma^l}} \\
&= \delta_{\beta, \delta} \sum_{x \in \Sigma^k} q(x) \hat{v}_\alpha^*(x) \hat{v}_\gamma^*(x)
\end{align*}
which is by Lemma \ref{lemma:pairing_final} equal to
\begin{align*}
&= \delta_{\beta, \delta} \langle v_\alpha, v_\gamma \rangle_{\mathscr{V}_k} = \delta_{\beta, \delta} \delta_{\alpha, \gamma}\,.
\end{align*}
This shows that the set is orthonormal and thus linearly independent. It is a spanning set since
\[
\dim \mathscr{H}_{k,l} = \dim(\mathbb{R}^{\Sigma^k}) \dim(\mathbb{R}^{\Sigma^l}) = |\Lambda_{k,l}||\Lambda_{k,l}^{++}|
\]
which completes the proof.
\end{proof}

Given a statistical model $p(y|x,w)$ for $x \in \Sigma^k, y \in \Sigma^l$ and $w \in W$ we denote by $f_w$ the corresponding function in $\mathscr{H}_{k,l}$
\begin{equation}
f_w(x) = \sum_{y \in \Sigma^l} p(y|x, w) y
\end{equation}
and we write $p(x,y|w) = p(y|x,w) q(x)$. Pairing an observable like $e_{\alpha \beta}$ with the function space representation of a probability distribution like $p$ gives an expectation $\mathbb{E}_p\big[ e_{\alpha\beta} \big]$ in the following sense:

\begin{lemma}\label{lemma:calc_comp_fw} For any $\alpha \in \Lambda, \beta \in \Lambda^{++}$ we have
\begin{equation}\label{eq:calc_comp_fw}
\big\langle f_w, e_{\alpha\beta} \big\rangle_{\mathscr{H}_{k,l}} = \sum_{x \in \Sigma^k,y \in \Sigma^l} p(x,y|w) e_{\alpha\beta}(x)(y) \,.
\end{equation}
\end{lemma}
\begin{proof}
We calculate
\begin{align*}
\big\langle f_w, e_{\alpha \beta} \big\rangle &= \sum_{x \in \Sigma^k} q(x) \big\langle f_w(x), e_{\alpha\beta}(x) \big\rangle_{\mathbb{R}^{\Sigma^l}} \\
&= \sum_{x \in \Sigma^k} q(x) \big\langle f_w(x), \hat{v}_\alpha^*(x) u_\beta \big\rangle_{\mathbb{R}^{\Sigma^l}} \\
&= \sum_{x \in \Sigma^k} q(x) \big\langle f_w(x), u_\beta \big\rangle_{\mathbb{R}^{\Sigma^l}} \hat{v}_\alpha^*(x) \\
&= \sum_{x \in \Sigma^k} \sum_{y \in \Sigma^l} p(y|x, w) \hat{v}_\alpha^*(x) u_\beta^*(y) q(x) \\
&= \sum_{x \in \Sigma^k} \sum_{y \in \Sigma^l} p(x, y|w) e_{\alpha\beta}(x)(y)\,.
\end{align*}
as claimed.
\end{proof}

The modes $\alpha$ are \emph{modes of variation} or ``patterns'' in the data distribution and we think of $e_{\alpha\beta}(x)(y)$ as the weighting of $xy$ in the mode, or more informally, as a measure of the degree to which $xy$ ``follows'' the pattern. Note it is immediate from \eqref{eq:decomp_ckl} that
\begin{equation}
\big\langle \mathcal{C}_{k,l}, e_{\alpha\beta}\big\rangle_{\mathscr{H}_{k,l}} = \delta_{\alpha, \beta} s_\alpha
\end{equation}

\subsection{Effective True Distribution}\label{section:effective_true}

In this section we discuss how to truncate the true distribution using the mode basis. After normalisation we obtain what we refer to as an \emph{effective} true distribution. Let $k,l > 0$ be fixed and set $\Lambda = \Lambda_{k,l}$, $\Lambda^+ = \Lambda_{k,l}^+$ and $\Lambda^{++} = \Lambda_{k,l}^{++}$. We have defined above an orthonormal basis $\{ e_{\alpha \beta} \}_{\alpha \in \Lambda, \beta \in \Lambda^{++}}$ for $\mathscr{H}_{k,l}$ (Lemma \ref{lemma:orthonormal_basis_H1}). Suppose that we put a total order on $\Lambda$ (for example putting the largest singular values first). Then for any $\chi \in \Lambda$ we can define
\[
\Lambda^{\le \chi} = \{ \alpha \in \Lambda \mid \alpha \le \chi \}\,, \qquad \Lambda^{+, \le \chi} = \{ \alpha \in \Lambda^+ \mid \alpha \le \chi \}\,.
\]
This defines a subspace
\[
\mathscr{H}^{\le \chi}_{k,l} := \operatorname{span}\Big\{ e_{\alpha \beta} \mid \alpha \in \Lambda^{\le \chi}, \beta \in \Lambda^{+,\le \chi} \Big\}\,.
\]
Note that this subspace is defined by linear equations in the coordinates of $\mathscr{H}_{k,l}$ since the $e_{\alpha \beta}$ are related to these coordinates by a change of basis. We let
\[
P^{\le \chi}: \mathscr{H}_{k,l} \longrightarrow \mathscr{H}_{k,l}^{\le \chi}
\]
denote the orthogonal projection onto this subspace. Recall that $\mathscr{P}_{k,l}$ denotes the subspace of functions taking values in probability distributions on $\Sigma^l$ (Definition \ref{defn:pkl}). We define
\[
\mathscr{P}^{\le \chi}_{k,l} = \mathscr{P}_{k,l} \cap \mathscr{H}^{\le \chi}_{k,l}\,.
\]
By construction this is an $m$-flat submanifold in the sense of \citet{amari2016information} since it is cut out by linear equations in the statistical manifold $\mathscr{P}_{k,l}$. The true distribution $q$ gives a point in $\mathscr{P}_{k,l}$ (or equivalently $\mathcal{C}_{k,l}$ does). Hence the following is well-defined:

\begin{definition}  \label{defn: truncated distribution}
We define
\[
q^{(\chi)} = \operatorname{argmin}_{p \in \mathscr{P}_{k,l}^{\le \chi}} D_{KL}(q \, \| \, p)
\]
to be the unique distribution in $\mathscr{P}_{k,l}^{\le \chi}$ which minises the KL divergence to the original distribution $q$. We refer to this as the \emph{effective} true distribution associated to the mode cutoff $\chi$.
\end{definition}
\begin{remark}
The uniqueness of the $q^{(\chi)}$ is proved in \citet{Amari2009}.
\end{remark}
\begin{remark} There is a more straightforward way to construct a truncation, but it has some subtleties. The \emph{truncation} of $\mathcal{C}_{k,l}: \mathscr{V}_k \longrightarrow \mathbb{R}^l$ defined by
\begin{equation}
\mathcal{C}^{\le \chi}_{k,l} = P^{\le \chi}\big( \mathcal{C}_{k,l} \big) = \sum_{\alpha \le \chi} s_\alpha u_\alpha \circ \hat{v}_\alpha^*.
\end{equation}
In the notation of propensities (Lemma \ref{lemma:mode_decomp_cond_prob}) we have for $x \in \Sigma^k, \tau \in \Sigma$
\begin{equation}
\mathcal{C}^{\le \chi}_{k,l}(x) = \sum_{y \in \Sigma^l} \sum_{\alpha \le \chi} q(y \mid x, \alpha) q(\alpha) \, y\,.
\end{equation}
Note however that $\mathcal{C}^{\le \chi}_{k+1}(x)(y)$ is not necessarily a probability distribution on $\Sigma^l$. Indeed these quantities can be negative. We can remove these negatives and normalise to produce a conditional probability distribution over $y$ given by
\begin{equation}
    q^{(\chi, b)}(y|x) := \frac{\max \big\{0, \sum_{\alpha \leq \chi} q(y|x, \alpha)q(\alpha) \big\}}{\sum_{y \in \Sigma^l} \max \big\{0, \sum_{\alpha \leq \chi} q(y|x, \alpha)q(\alpha) \big\} }\,.
\end{equation}
This has the advantage of being simpler to describe explicitly than $q^{(\chi)}$ but is theoretically less well-justified. We do not consider the distribution $q^{(\chi, b)}$ further in this paper.
\end{remark}

\begin{remark}
Here we have treated a single pair $k, l$ at a time. But we note that $q$ satisfies the language condition (Definition \ref{defn:language}) which ties the $q_k$ for all $k$ to one another. We can define simultaneous truncations of all $q_k$ by a choice of mode cutoff $\chi^{k,l}$ for each $k,l$ since this leads to an $m$-flat submanifold of the statistical manifold of all languages.
\end{remark}

\section{Examples}\label{section:examples}

The purpose of this section is to briefly argue that the singular values of the SVD of the fundamental tensor, arranged as a linear transformation $\mathcal{C}_{k,l}$, is a reasonable notion of ``pattern'' in sequence models, including tokenised natural language. First we exhibit some simple abstract modes, before giving empirical examples. We note that tensor decompositions are a common technique for representing structure in natural language; see \citep{anandkumar2012method} and the references therein.

\subsection{Theoretical}

\begin{definition} For $k, l \in \mathbb{Z}_{>0}$, an \emph{absolute bigram} is a sequence $xy$ where $x \in \Sigma^k$ and $y \in \Sigma^l$ with the property that
\begin{itemize}
    \item[(a)] $q(xz) = 0$ for all $z \in \Sigma^l$ with $z \neq y$.
    \item[(b)] $q(ty) = 0$ for all $t \in \Sigma^k$ with $t \neq x$.
\end{itemize}
\end{definition}

\begin{lemma}\label{lemma:absolute_bigram} If $xy$ is an absolute bigram, then $q(x)^{1/2} x$ is a right singular vector with corresponding left singular vector $y$ and singular value $s = q(x)^{1/2}$, of $\mathcal{C}_{k, l} \in \mathscr{H}_{k, l} $.
\end{lemma}
\begin{proof}
Let $xy$ be an absolute bigram. Then it follows from the definition of absolute bigrams that $q(y|x) = 1$. Note that $q(z|x) = 0$ if and only if $q(xz) = 0$, and $q(y |t) = 0$ if and only if $q(ty) = 0$. Note that by Remark \ref{remark:normalised_symbol}, $q(x)^{1/2} x$ is a unit vector in $\mathscr{V_k}$ and $y$ is a unit vector in $\mathbb{R}^{\Sigma^l}$ which are respectively the domain and codomain (as inner product spaces) of $\mathcal{C}_{k,l}$. Then
\begin{align*}
\mathcal{C}_{k,l}(q(x)^{1/2}x) &= q(x)^{1/2} \sum_{z} q(z|x) z = q(y|x) q(x)^{1/2} y = q(x)^{1/2} y \\
\mathcal{C}_{k,l}^*(y) &= \sum_{t} q(y|t) t = q(y|x) x = x\,.
\end{align*}
Hence by Remark \ref{remark:compute_dagger}, $\mathcal{C}_{k,l}^\dagger(y) = q(x) x = q(x)^{1/2} ( q(x)^{1/2} x )$ as claimed in (a). 
\end{proof}

\begin{remark}\label{remark:absolute_bigram_propensities}
In the terminology of Definition \ref{defn:propensity}, with $\alpha$ denoting the mode index corresponding to an absolute bigram $x y$, the propensities are $q(\alpha) = q(x) = q(xy)$ and
\[
    q(z|t, \alpha) = q(x)^{-1/2} y^*(z) q(x)^{1/2} \hat{x}^*(t) = \delta_{z, y}  \hat{x}^*(t) = q(x)^{-1}\delta_{z, y}\delta_{t, x} \,.
\]
\end{remark}

\begin{lemma}\label{lemma:modes_model} If $x y$ is an absolute bigram with mode index $\alpha$ then $\langle f_w, e_{\alpha \alpha} \rangle_{\mathscr{H}} = p(xy|w) q(x)^{-1/2}$.
\end{lemma}
\begin{proof}
By Lemma \ref{lemma:absolute_bigram} we have $u_\alpha = y, v_\alpha = q(x)^{1/2} x$ so by Lemma \ref{lemma:calc_comp_fw} we have
\begin{align*}
\langle f_w, e_{\alpha \alpha} \rangle_{\mathscr{H}} &= p(x y|w) \hat{v}_\alpha^*(x) = p(x y|w) \langle q(x)^{1/2} x, x \rangle_{\mathscr{V}_k} = p(x y|w)q(x)^{-1/2}
\end{align*}
as claimed.
\end{proof}

While absolute bigrams may be rare, there is a generalisation which is more common:

\begin{definition}\label{defn:collective_bigram} A \emph{collective bigram} is a pair $(S, y)$ consisting of a nonempty set $S \subseteq \Sigma^k$ and $y \in \Sigma^l$ satisfying
\begin{itemize}
    \item[(a)] $q(y|s) = 1$ for all $s \in S$.
    \item[(b)] $q(y|t) = 0$ if $t \notin S$.
\end{itemize}
\end{definition}

As above we can prove collective bigrams give modes:

\begin{lemma} If $(S,y)$ is a collective bigram then $v = \frac{q(y)^{1/2}}{|S|^{1/2}} \sum_{s \in S} s$ is a right singular vector with corresponding left singular vector $y$ and singular value $q(y)^{1/2} |S|^{1/2}$ of $\mathcal{C}_{k,l}: \mathscr{V}_k \longrightarrow \mathbb{R}^{\Sigma^l}$.
\end{lemma}
\begin{proof}
We have $\mathcal{C}_{k,l}(s) = y$ for $s \in S$, hence $\mathcal{C}_{k,l}(v) = q(y)^{1/2} |S|^{1/2} y$. On the other hand $\mathcal{C}_{k,l}^*(y) = \sum_{s \in S} q(y|s) s = \sum_{s \in S} s$ and hence by Remark \ref{remark:compute_dagger}
\[
\mathcal{C}_{k,l}^\dagger(y) = q(y) \mathcal{C}_{k,l}^*(y) = q(y) \sum_{s \in S} s = q(y) \frac{|S|^{1/2}}{q(y)^{1/2}} v = q(y)^{1/2} |S|^{1/2} v\,.
\]
This proves the claim.
\end{proof}

\subsection{Empirical}

For various $k,l$ we collect matrices of empirical conditional probabilities $p(y|x)$ for token sequences in the Pile \citep{gao2020pile}, see Appendix \ref{section:token_svd}. Note that this is a large sparse matrix, and we do not normalise the columns. Nonetheless the SVD of these matrices is an indication of the kinds of patterns that we can expect to see as right and left singular vectors.

\begin{example}\label{example:k1_l1} For $k = l = 1$ the top singular value is $33.89$ with left singular vector (ordered by absolute value of loading)
\begin{equation}\label{eq:v_k1_l1}
u = 30.84 \tokenbox{,} + 13.20 \tokenbox{.} + 2.77 \tokenbox{of} + 2.39 \tokenbox{and} + \cdots
\end{equation}
and the right singular vector is
\[
v = \kappa \tokenbox{370} + \kappa \tokenbox{Additionally} + \kappa \tokenbox{Similarly} + \kappa \tokenbox{sequently} + \cdots
\]
where $\kappa = 0.0268$. See Figure \ref{fig:k1_l1_example1}.
\end{example}

\begin{example}\label{example:k2_l1} For $k = 2$, $l = 1$ the seventh top singular value is $46.18$ with left singular vector
\begin{equation}
u = 46.15 \tokenbox{to} - 1.58 \tokenbox{and} - 0.53 \tokenbox{of} + \cdots
\end{equation}
and the right singular vector is
\[
v = \kappa \tokenbox{~in}\tokenbox{~response} + \kappa \tokenbox{~in}\tokenbox{~order} + \kappa \tokenbox{were}\tokenbox{~subjected} + \cdots
\]
where $\kappa = 0.0216$. See Figure \ref{fig:k2_l1_example1}.
\end{example}

\begin{figure}[tbp]
    \centering
    \includegraphics[width=\textwidth]{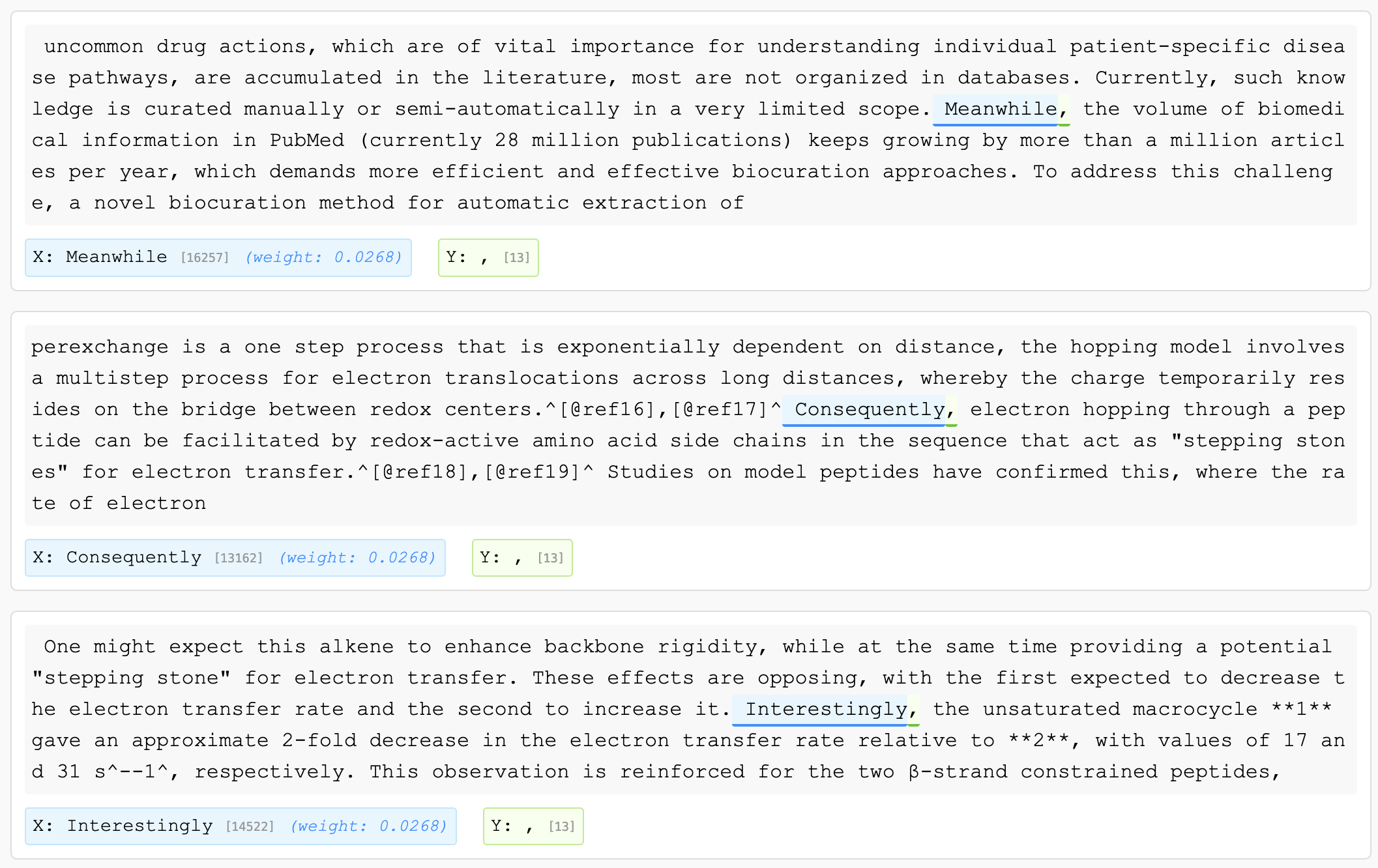}
    \caption{\textbf{Empirical modes.} We show an example of the $x,y$ pair which are heavily loaded in the first empirical mode for $k = 1$, $l = 1$ in our experiments on the Pile. Shown are three text samples. Next to each $X, Y$ token we show the index in the tokeniser.}
    \label{fig:k1_l1_example1}
\end{figure}

\begin{figure}[tbp]
    \centering
    \includegraphics[width=\textwidth]{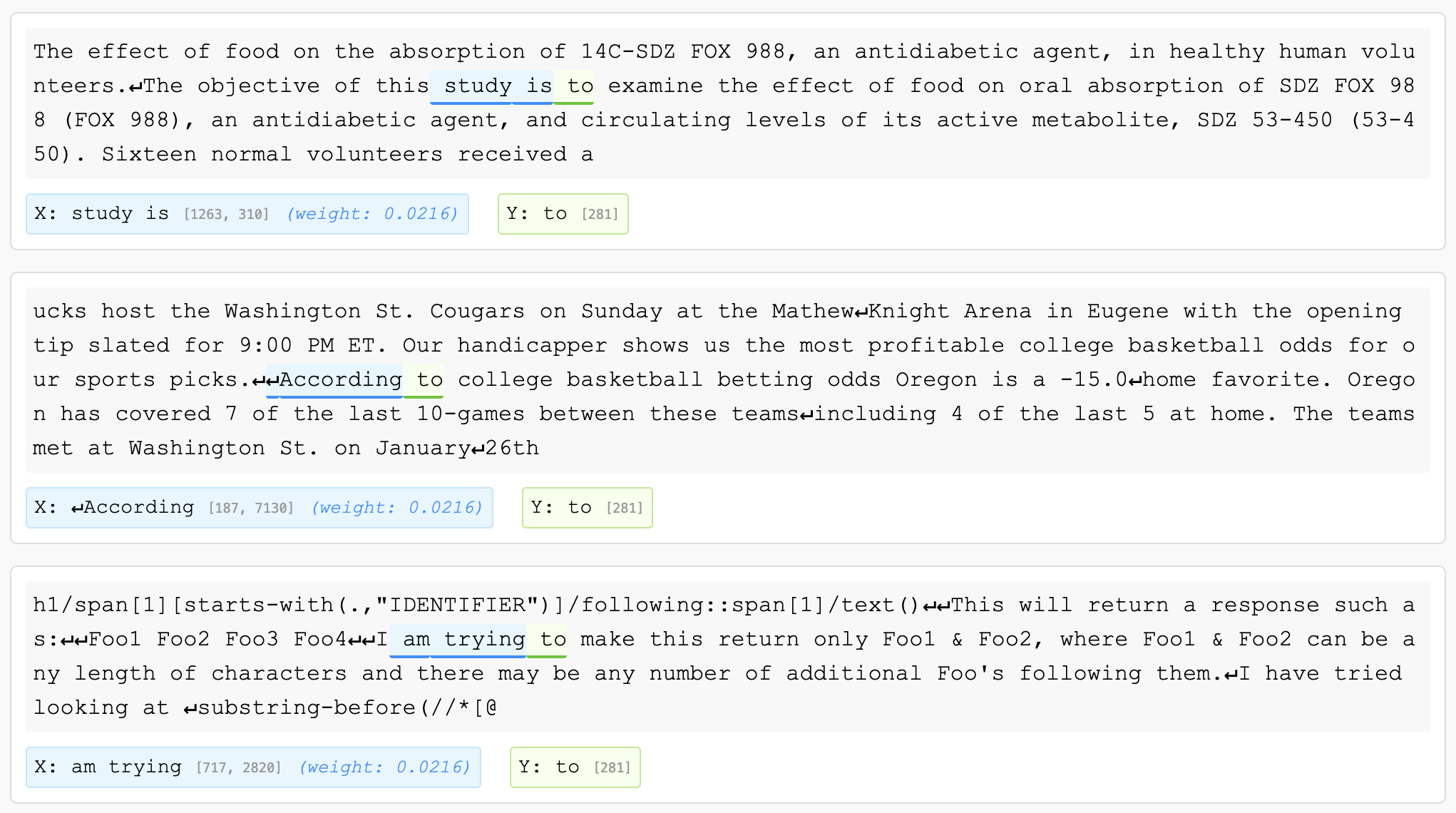}
    \caption{\textbf{Empirical modes.} We show an example of the $x,y$ pair which are heavily loaded in the first empirical mode for $k = 2$, $l = 1$ in our experiments on the Pile. Shown are three text samples. Next to each $X, Y$ token we show the index in the tokeniser.}
    \label{fig:k2_l1_example1}
\end{figure}

\begin{figure}[tbp]
    \centering
    \includegraphics[width=\textwidth]{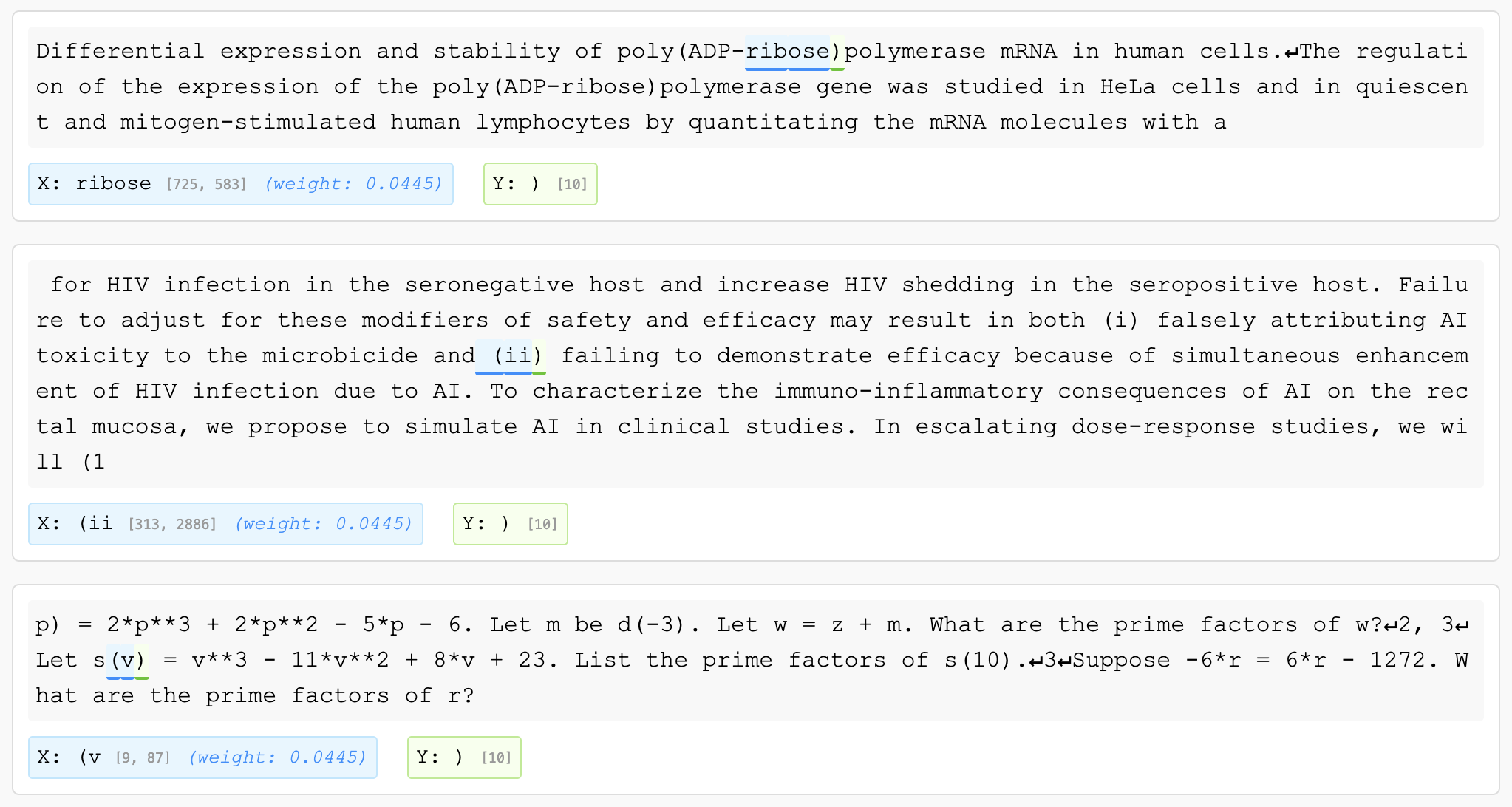}
    \caption{\textbf{Empirical modes.} We show an example of the $x,y$ pair which are heavily loaded in the first empirical mode for $k = 2$, $l = 1$ in our experiments on the Pile. Shown are three text samples. Next to each $X, Y$ token sequence we show the indices in the tokeniser.}
    \label{fig:k2_l1_example2}
\end{figure}

We often observe that the right singular vector has a large number of $x$ token sequences with the same largest ``loading'' (i.e. coefficient in the right singular vector) whereas there is a clear dominating token sequence $y$ (and sometimes as in \eqref{eq:v_k1_l1} a second) and then a long tail of token sequences with coefficients about an order of magnitude smaller. We can think of these as some perturbation of the collective bigrams of Definition \ref{defn:collective_bigram}.

\begin{example}\label{example:k2_l1_example2} For $k = 2$, $l = 1$ component $16$ has singular value $22.45$ with left singular vector
\begin{equation}
u = 22.44 \tokenbox{)} + 0.57 \tokenbox{).} + 0.28 \tokenbox{),} + \cdots
\end{equation}
and the right singular vector is
\[
v = \kappa \tokenbox{rib}\tokenbox{ose} + \kappa \tokenbox{~(}\tokenbox{ii} + \kappa \tokenbox{(}\tokenbox{v} + \cdots
\]
where $\kappa = 0.0445$. See Figure \ref{fig:k2_l1_example2}. Closing brackets and other punctuation and structural features are common in the SVD.
\end{example}

\begin{example}\label{example:k3_l3_example1} For $k = 3$, $l = 3$ component $28$ has singular value $5.58$ with left singular vector
\begin{equation}
u = 5.58 \tokenbox{,}\tokenbox{~and}\tokenbox{~the} + 0.05 \tokenbox{.}\tokenbox{~the}\tokenbox{~database} + \cdots
\end{equation}
and the right singular vector is
\[
v = \kappa \tokenbox{,}\tokenbox{~bright}\tokenbox{field} + \kappa \tokenbox{population}\tokenbox{~is}\tokenbox{~383} + \cdots
\]
where $\kappa = 0.1786$. See Figure \ref{fig:k3_l3_example1}.
\end{example}

\begin{example}\label{example:k3_l3_example2} Many modes involve newlines in various combinations with punctuation tokens. For $k = 3$, $l = 3$ component $37$ has singular value $5.05$ with left singular vector
\begin{equation}
u = 5.05 \tokenbox{:}\tokenbox{\\n}\tokenbox{\\n} + 0.08 \tokenbox{http}\tokenbox{://}\tokenbox{www}
\end{equation}
and the right singular vector is a linear combination of many token sequences. See Figure \ref{fig:k3_l3_example2}.
\end{example}

For more examples see Appendix \ref{appendix:more_modes},

\begin{remark} In many of these examples we see modes for $k,l$ (such as that in Example \ref{example:k3_l3_example2}) which are arguably ``descended'' from modes with $k' < k$ or $l' < l$. It would be interesting to develop this idea of hierarchical structure of modes across $k,l$ but we leave that to future work.
\end{remark}

\section{Mode Insensitivity}\label{section:mode_insensitivity}

We develop two technical hypothesis which express that in a region of parameter space a model is \emph{insensitive} to the information in higher modes. Consider a parametric model with parameters $w \in W$ defining conditional probabilities $p(y|x,w)$ for $x \in \Sigma^k$ and $y \in \Sigma^l$. In this paper we assume $W$ is compact and that $p(y|x,w)$ is continuously differentiable with respect to $W$ for all $x,y$.

We fix $k, l > 0$ and set $\mathscr{H} = \mathscr{H}_{k,l}$. We let $w_1,\ldots,w_d$ denote local coordinates at $w^* \in W$, so that the tangent space $T_{w^*}(W)$ is spanned by partial derivatives $\frac{\partial}{\partial w_i}$. 

\begin{definition} The log probability map $\Phi: W \longrightarrow \mathscr{H}$ is defined by
\[
\Phi(w)(x) = \sum_{y \in \Sigma^l} \log p(y|x,w) \cdot y
\]
\end{definition}

The Jacobian of $\Phi$ at a point $w^*$ is a linear transformation. Since $\mathscr{H}$ is a (finite-dimensional) affine space we identify it with its tangent map at a given point.

\begin{definition}
The Jacobian of the log probability map is
\[
\Psi_{w^*}: T_{w^*}(W) \longrightarrow \mathscr{H}
\]
which sends each tangent vector to a function representing the corresponding direction of change in log-probabilities:
\begin{equation}
\Psi_{w^*}\left(\frac{\partial}{\partial w_i}\right)(x) = \sum_{y \in \Sigma^l} \frac{\partial}{\partial w_i} \log p(y|x,w^*) \cdot y
\end{equation}
where we view $y \in \Sigma^l$ as a basis vector in $\mathbb{R}^{\Sigma^l}$.
\end{definition}

Any linear transformation from $T_{w^*}(W)$ to $\mathscr{H}$ can be viewed as a $\mathscr{H}$-valued cotangent vector and so sometimes we will write the Jacobian of the log probability map simply as
\begin{equation}
\sum_{y \in \Sigma^l} \nabla_w \log p(y|x,w^*) \cdot y
\end{equation}
where we assume $W \subseteq \mathbb{R}^d$ is given the canonical metric on its tangent space.

The image of $\Psi_{w^*}$ characterizes which directions in function space the model can immediately move toward during gradient-based optimisation. If a particular mode in the data distribution lies outside this image, the model will struggle to capture it efficiently during learning.

We think of the conditional distributions of $q,q'$ as vectors in $\mathscr{H}$ as in Section \ref{section:tensor_decomp}, and abuse notation by writing $q$ for $\mathcal{C}_{k,l}$. Given the vector $q - q'$ in $\mathscr{H}$ the pairing with $\Psi_{w^*}$ is a linear functional on $T_{w^*}(W)$ defined by sending a tangent vector $\nu$ to
\begin{align}
\big\langle \Psi_{w^*}(\nu), q - q' \big\rangle_{\mathscr{H}} &= \sum_{x \in \Sigma^k} \big\langle \Psi_{w^*}(\nu)(x), q(-|x) - q'(-|x) \big\rangle_{\mathbb{R}^l} q(x) \nonumber\\
&= \sum_{x \in \Sigma^k} \sum_{y \in \Sigma^l} \big(q(x,y) - q'(x,y)\big) \sum_{i=1}^d \nu_i \frac{\partial}{\partial w_i} \log p(y|x,w^*)\,. \label{eq:gradient_insens_defn}
\end{align}
using the $L^2$-norm on $T^{w^*}(W) = T_{w^*}(W)^*$ the norm of this is the $L^2$-norm of the vector
\[
\big\Vert \big\langle \Psi_{w^*}(-), q - q' \big\rangle_{\mathscr{H}} \big\Vert_2 := \bigg\| \sum_{x \in \Sigma^k} \sum_{y \in \Sigma} \big( q(x, y) - q'(x, y) \big) \nabla_w \log p(y|x, w) \bigg\|_2\,.
\]
Similarly we can pair $\Phi(w)$ with $q - q'$
\begin{equation}\label{eq:log_prob_insens_defn}
\big\langle \Phi(w), q - q' \big\rangle_{\mathscr{H}} = \sum_{x \in \Sigma^k} \sum_{y \in \Sigma^l} \big(q(x,y) - q'(x,y)\big) \log p(y|x,w)\,.
\end{equation}

In the following two definitions the \emph{model} means $p(y|x,w)$ and implicitly we mean \emph{insensitivity to the difference between $q,q'$}.

\begin{definition}\label{defn:gradient_insensitive} The model is \emph{gradient-insensitive} for the constant $A$ at $w^*$ if
\begin{equation}
\big\Vert \big\langle \Psi_{w^*}(-), q - q' \big\rangle_{\mathscr{H}} \big\Vert_2 < A\,.
\end{equation}
We say this condition hold in some set $W' \subseteq W$ if it holds at every point.
\end{definition}

\begin{definition}\label{defn:insensitive} The model is \emph{log-probability-insensitive} for the constant $B$ at $w^*$ if
\begin{equation}
\Big| \big\langle \Phi(w^*), q - q' \big\rangle_{\mathscr{H}} \Big| < B\,.
\end{equation}
We say this condition hold in some set $W' \subseteq W$ if it holds at every point.
\end{definition}

\subsection{Motivation for Gradient-Insensitivity}

The intuitive content of the gradient-insensitivity condition is quite straightforward, and can be summarised by the phrase ``can't learn''. To say that $\langle \Psi_{w^*}(\nu), q - q' \rangle_{\mathscr{H}}$ is small is to say that it is hard to change the model's log probabilities in directions in the space of predictions of $y$ given $x$, which are aligned with the difference between $q, q'$. 

For example suppose that $p(y|x,w)$ is computed by a neural network. This model has a finite capacity to memorize absolute bigrams, so if $q$ contains more absolute bigrams than $p$ has capacity, and $q'$ represents truncating the true distribution at some $\chi$ that is past this capacity, it is reasonable to suppose that $p$ has no gradient for learning the difference $q - q^{(\chi)}$. 

More generally we follow \cite{rogers2004semantic} in supposing that before models have learned coarse distinctions between categories it is hard to learn distinctions between finer categories; this hierarchical structure in the data distribution means that is reasonable to suppose the gradient for learning higher modes is small before lower modes are learned.

\subsection{Motivation for Log-Probability Insensitivity}\label{section:motivation_log_prob_insens}

The motivation for log-probability-insensitivity is more subtle.

Naively to make $\big\langle \Phi(w), q - q' \big\rangle_{\mathscr{H}}$ small in norm we would have to arrange that $\log p(y|x,w)$ is not too large whenever $q(x,y)$ is different from $q'(x,y)$, for example in the case where $q'$ is obtained by mode truncation, whenever the prediction of $y$ from $x$ depends non-trivially on higher modes. It is not immediately clear that this is reasonable, and so we need a more careful analysis.

Note that using the basis of $\mathscr{H}$ and writing $\langle - \rangle$ for $\langle - \rangle_{\mathscr{H}}$ we have
\begin{equation}\label{eq:phiqqprimebound}
\big\langle \Phi(w), q - q' \big\rangle = \sum_{\alpha \in \Lambda, \beta \in \Lambda^{++}} \big\langle \Phi(w), e_{\alpha \beta} \big\rangle \big\langle e_{\alpha \beta}, q - q' \big\rangle\,.
\end{equation}
Suppose that $q' = q^{(\chi)}$ is produced by mode truncation as in Section \ref{section:effective_true}. Then the idea is that the support of $q - q^{(\chi)}$ is mostly in the modes $> \chi$. To be more precise, $q^{(\chi)}$ is the closest point of $\mathscr{P}^{\le \chi}_{k,l}$ to $q$, and this may be different from the projection $P^{\le \chi}(q)$ which is a vector in $\mathscr{H}^{\le \chi}$ but need not be a distribution. However it is reasonable to assume that $q^{(\chi)}$ is not \emph{too} different from $P^{\le \chi}(q)$ if $\chi$ is large, since both operations should only have a small effect on most sequences.

Notice that for $\alpha, \beta \le \chi$
\begin{align}
\big\langle e_{\alpha \beta}, q - q^{(\chi)} \big\rangle &= \big\langle e_{\alpha \beta}, q - P^{\le \chi}(q) \big\rangle + \big\langle e_{\alpha \beta}, P^{\le \chi}(q) - q^{(\chi)} \big\rangle\nonumber\\
&= \big\langle e_{\alpha \beta}, P^{\le \chi}(q) - q^{(\chi)} \big\rangle\,. \label{eq:alpha_beta_pchi}
\end{align}
By Parseval's identity
\[
\Vert P^{\le \chi}(q) - q^{(\chi)} \Vert^2 = \sum_{\alpha \le \chi, \beta \le \chi} \Big| \big\langle e_{\alpha \beta}, P^{\le \chi}(q) - q^{(\chi)} \big\rangle \Big|^2
\]
and so making this small, which amounts to taking $\chi$ sufficiently large, makes all the terms in \eqref{eq:alpha_beta_pchi} simultaneously small. More formally
\begin{align}
\big| \big\langle \Phi(w), q - q' \big\rangle \big| &\le \Big| \sum_{\alpha \le \chi, \beta \le \chi} \big\langle \Phi(w), e_{\alpha \beta} \big\rangle \big\langle e_{\alpha \beta}, q - q' \big\rangle \Big|\nonumber\\
&\qquad + \Big| \sum_{\alpha > \chi \text{ or } \beta > \chi} \big\langle \Phi(w), e_{\alpha \beta} \big\rangle \big\langle e_{\alpha \beta}, q - q' \big\rangle \Big|\,. \label{eq:cauchy_schwartz}
\end{align}
By Cauchy-Schwartz the first term is bounded by
\[
\Vert \Phi(w) \Vert^2 \Vert P^{\le \chi}(q) - q^{(\chi)} \Vert^2\,.
\]
Let's suppose this is $< A/2$. To bound the second summand in \eqref{eq:cauchy_schwartz} note that
\begin{equation}
\sum_{\alpha > \chi \text{ or } \beta > \chi} \big\langle \Phi(w), e_{\alpha \beta} \big\rangle \big\langle e_{\alpha \beta}, q \big\rangle\, = \sum_{\alpha > \chi \text{ or } \beta > \chi} \big\langle \Phi(w), e_{\alpha \beta} \big\rangle \delta_{\alpha \beta} s_\alpha \nonumber = \sum_{\alpha > \chi } \big\langle \Phi(w), e_{\alpha \alpha} \big\rangle s_\alpha\,.\label{eq:insensitive_bound_needed}
\end{equation}
Now
\begin{align*}
\big\langle \Phi(w), e_{\alpha \alpha} \big\rangle = \sum_{x \in \Sigma^k} \sum_{y \in \Sigma^l} e_{\alpha \alpha}(x)(y) \log p(y|x,w) q(x)\,.
\end{align*}
To give a plausibility argument for bounding this term we now make the assumption that all the $\alpha$'s are absolute bigrams. Then $s_\alpha = q(\sigma \tau)^{1/2}$ so
\[
\big\langle \Phi(w), e_{\alpha \alpha} \big\rangle s_\alpha = q(\sigma)^{-1} \log p(\tau | \sigma, w) q(\sigma) q(\sigma \tau)^{1/2} = \log p(\tau|\sigma, w) q(\sigma \tau)^{1/2}\,.
\]
To bound \eqref{eq:cauchy_schwartz} by $A$ we need
\[
\log p(\tau | \sigma, w) < \frac{A}{2 q(\sigma \tau)^{1/2} C}
\]
where $C$ is the number of modes $> \chi$. The total number of modes is $|\Sigma|^l$. Here we make an assumption on our data distribution, namely, that for sufficiently long sequences $x \in \Sigma^z$ we have $q(x) \le 2^{-zH}$ for some entropy rate $H$ with high probability \citep[\S 3.1]{Thomas2006}. Then since $q(\sigma \tau)$ is a sequence of length $k + l$ if the entropy rate is $H$ we have $q(\sigma \tau) \le 2^{-(k+l)H}$. Hence
\begin{align*}
q(\sigma \tau)^{1/2} C &\le 2^{-\frac{(k+l)}{2} H} |\Sigma|^l\\
&= 2^{-\frac{k+l}{2} H} 2^{l \log_2|\Sigma|}\\
&= 2^{l \log_2|\Sigma| - \frac{k+l}{2} H}
\end{align*}
and so
\begin{equation}\label{eq:varepsilon_bound_insens}
\frac{A}{2q(\sigma \tau)^{1/2} C} \ge 2^{\frac{k+l}{2} H - l \log_2|\Sigma| - 1} A
\end{equation}
So to recap, in order to get log-probability-insensitivity for $A$ we need $\Vert P^{\le \chi}(q) - q^{(\chi)} \Vert$ to be small enough relative to $\Vert \Phi(w) \Vert$ and in addition it would suffice (if all the modes were absolute bigrams) to have
\begin{equation}\label{eq:bound_ok_or_not}
\log p(\tau | \sigma, w) < 2^{\frac{k+l}{2} H - l \log_2|\Sigma| - 1} A
\end{equation}
for all these bigrams. Note that this is trivially true if the model perfectly predicts all the bigrams corresponding to modes $> \chi$ since then these log probabilities are zero. However we do not want to rely on this. We would prefer to have \eqref{eq:bound_ok_or_not} be a mild constraint because the right hand side is large, and have most of the pressure of the log-probability-insensitivity be on $\Vert P^{\le \chi}(q) - q^{(\chi)} \Vert$ being small.

Let us substitute some reasonable numbers. If $|\Sigma| = 65k$ then $\log_2|\Sigma| = 16$ and we take $H \in [1,10]$. Suppose $H = 1$ (this is the pessimistic case since it leads to slower decay in probabilities). Then the exponent is $\frac{k}{2} - \frac{31 l}{2} - 1$. Thus provided
\[
k > 2 + 31 l
\]
the bound in \eqref{eq:bound_ok_or_not} depends more on $A$ than it does on factors to do with the data distribution or $k,l$. For next-token prediction $l = 1$ and modern transformer language models are typically trained with context lengths $k$ that are tens or hundreds of thousands of tokens, so this seems reasonable.

\section{Modes and LLC estimation}\label{section:neg}

As in Section \ref{section:effective_true} we fix $k, l > 0$ and set $\Lambda = \Lambda_{k,l}$, $\mathscr{H} = \mathscr{H}_{k,l}$. We consider a parametric model with parameters $w \in W$ defining conditional probabilities $p(y|x,w)$ for $x \in \Sigma^k$ and $y \in \Sigma^l$. We assume $W \subseteq \mathbb{R}^d$ is compact with coordinates $w_1,\ldots,w_d$ and that $p(y|x,w)$ is continuously differentiable with respect to $W$ for all $x,y$. In the following $W' \subseteq W$ denotes a compact set. We assume throughout that $W'$ is large enough that all SGLD chains lie inside this set.

We consider the effective true distribution $q^{(\chi)}$ associated to truncating at a mode $\chi \in \Lambda$. Associated to the two distributions $q, q^{(\chi)}$ we have two empirical negative log likelihoods
\begin{align}
    L_m(w) &= -\frac{1}{m} \sum_{i=1}^m \log p(y_i | x_i, w) \,,\label{eq:L_m_sequence_mod}\\
    L_m^{(\chi)}(w) &= -\frac{1}{m} \sum_{i=1}^m \log p\big(y^{(\chi)}_i | x^{(\chi)}_i, w\big) \,,
\end{align}
where $\{(x_i, y_i)\}_{i=1}^m$ and $\big\{(x_i^{(\chi)},
y_i^{(\chi)})\big\}_{i=1}^m$ are sampled from the true distribution $q(x, y)$ of sequences of length $k + l$ and the truncated distribution $q^{(\chi)}(x, y)$ (\Cref{defn: truncated distribution}), respectively. Their theoretical counterparts are
\begin{align}
    L(w) &= - \sum_{x \in \Sigma^k} \sum_{y \in \Sigma^l} \log p(y|x, w) q(y|x) q(x)\,,\\
    L^{(\chi)}(w) &= - \sum_{x \in \Sigma^k} \sum_{y \in \Sigma^l} \log p(y|x, w) q^{(\chi)}(y|x) q(x)\,.
\end{align}
We first enumerate the hypotheses we require on the model $p(y|x,w)$ for our main theorem. If $W' \subseteq W$ is compact then define
\[  
    F(x, y) = \sup_{w \in W'} \|\nabla_w \log p(y|x, w)\|_2 < \infty
\]
for each $x \in \Sigma^k$ and $y \in \Sigma^l$ which is integrable as $\Sigma^k$ and $\Sigma^l$ are finite sets. Then the family 
\[
    \mathcal{F} = \{ - \nabla_w\log p(y|x, w) \, | \, w \in W'\}
\]
of functions on $x \in \Sigma^k$, $y \in \Sigma^l$ has a integrable envelope function $F(x, y)$: $\|-\nabla_w \log p(y|x, w) \|_2 < F(x, y)$ for all $x \in \Sigma^k$, $y \in \Sigma^l$, and $w \in W'$. 
Then it follows from \citet[\S 19.2]{Vaart_1998} that the family $\mathcal{F}$ is Glivenko-Cantelli, which, by definition, says that
\begin{equation}
    \sup_{w \in W'} \|\nabla_w L_m(w) - \nabla_w L(w) \|_2 \overset{P}{\longrightarrow} 0 \,, \text{\, as $m \to \infty$}.
\end{equation}
We further assume that $p(y|x, w)$ is bounded away from zero for all $x \in \Sigma^k$, $y \in \Sigma^l$, and $w \in W'$ by a fixed constant. Then it follows that the family
\[
    \mathcal{G} = \{-\log p(y|x, w) \, | \, w \in W'\}
\]
of functions on $x \in \Sigma^k$ and $y \in \Sigma^l$ is has a constant envelope function: there exists $G \in \mathbb{R}_{>0}$ such that $|-\log p(y|x, w)| < G$ for all $x \in \Sigma^k$, $y \in \Sigma^l$, and $w \in W'$. Then it follows from \citet[\S 19.2]{Vaart_1998} that the family $\mathcal{G}$ is Glivenko-Cantelli, which, by definition, says that
\begin{equation}
    \sup_{w \in W'}  | L_n(w) -  L(w) | \overset{P}{\longrightarrow} 0 \,, \text{\, as $n \to \infty$}.
\end{equation}

\begin{lemma}\label{lemma: empirical version of negligible modes} Suppose that the model is gradient-insensitive in $W'$ to the difference between $q$ and $q^{(\chi)}$ for some positive constant $A$. Then for any $\xi$, $\delta > 0$, there exists a large enough $m$ such that
\begin{equation}
    \mathbb{P}\bigg( \big\| \nabla_w L_m(w) - \nabla_w L_m^{(\chi)}(w) \big\|_2 < A + \xi \bigg) > 1 - \delta \, ,
\end{equation}
for all $w \in W'$.
\end{lemma}
\begin{proof}
By triangle inequality, we have
\begin{align*}
    \big\| \nabla_w L_m(w) - \nabla_w L_m^{(\chi)}(w) \big\|_2 &\leq \big\| \nabla_w L_m(w) - \nabla_w L(w) \big\|_2 + \big\| \nabla_w L(w) - \nabla_w L^{(\chi)}(w) \big\|_2 \nonumber \\
    & \qquad+ \big\| \nabla_w L^{(\chi)}(w) - \nabla_w L_m^{(\chi)}(w) \big\|_2 \,.
\end{align*}
So
\begin{align*}
    & \quad\mathbb{P}\bigg( \big\| \nabla_w L_m(w) - \nabla_w L_m^{(\chi)}(w) \big\|_2 < A + \xi \bigg) \\
    &\geq \mathbb{P} \bigg(\big\| \nabla_w L_m(w) - \nabla_w L(w) \big\|_2 + \big\| \nabla_w L(w) - \nabla_w L^{(\chi)}(w) \big\|_2 
     \\
     &\qquad \qquad+ \big\| \nabla_w L^{(\chi)}(w) - \nabla_w L_m^{(\chi)}(w) \big\|_2 < A + \xi\bigg) \\
    &\ge \mathbb{P} \bigg(\big\| \nabla_w L_m(w) - \nabla_w L(w) \big\|_2 < \frac{\xi}{2}\bigg) + \mathbb{P} \bigg(\big\| \nabla_w L(w) - \nabla_w L^{(\chi)}(w) \big\|_2 < A \bigg) \\
    & \qquad+ \mathbb{P}\bigg( \big\| \nabla_w L^{(\chi)}(w) - \nabla_w L_m^{(\chi)}(w) \big\|_2 < \frac{\xi}{2}  \bigg) + 1 - 3\,.
\end{align*}

Since $\nabla_w L_m(w) \to L(w)$ and $\nabla_w L_m^{(\chi)}(w) \to L^{(\chi)}(w)$ uniformly as $m \to \infty$, for $\frac{\delta}{3}$ and $\frac{\xi}{2}$, there exists large enough $m$ such that 
\begin{equation*}
    \mathbb{P} \bigg(\big\| \nabla_w L_m(w) - \nabla_w L(w) \big\|_2 < \frac{\xi}{2}\bigg) > 1 - \frac{\delta}{3} \,,
\end{equation*}
and 
\begin{equation*}
    \mathbb{P} \bigg(\big\| \nabla_w L^{(\chi)}(w) - \nabla_w L_m^{(\chi)}(w) \big\|_2 < \frac{\xi}{2}\bigg) > 1 - \frac{\delta}{3} \,.
\end{equation*}

By the hypothesis of gradient-insensitivity we have that
\begin{align*}
    \big\| \nabla_w L(w) - \nabla_w L^{(\chi)}(w) \big\|_2 &= \bigg\| \sum_{x \in \Sigma^k} \sum_{y \in \Sigma} \big[q^{(\chi)}(x, y) - q(x, y) \big] \nabla_w \log p(y|x, w) \bigg\|_2
\end{align*}
is less than $A$. In particular, 
\begin{equation*}
    \mathbb{P}\bigg( \big\| \nabla_w L(w) - \nabla_w L^{(\chi)}(w) \big\|_2 < A\bigg) = 1 > 1 - \frac{\delta}{3} \, .
\end{equation*}
Therefore,
\begin{align*}
    \mathbb{P}\bigg( \big\| \nabla_w L_m(w) - \nabla_w L_m^{(\chi)}(w) \big\|_2 < A + \xi \bigg) &> 3\bigg( 1 - \frac{\delta}{3}\bigg) + 1 - 3 \\
    &= 1 - \delta
\end{align*}
as claimed.
\end{proof}

\begin{remark}
The need to potentially take $m$ large in the lemma is a weakness in our technical treatment. The idea of SGLD is that $m \ll n$ but since we average over many SGLD steps the effect of the mini-batch gradients approaches the effect of Langevin dynamics with $\nabla_w L_n$.
\end{remark}

As above $T$ is the length of the SGLD chains, $\epsilon_{\max} = \max\{\epsilon_t\}_{t=1}^T$, $\epsilon_{\min} = \min\{\epsilon_t\}_{t=1}^T$ are the maximum and minimum step-sizes, and $m, n, \beta, \gamma$ are the hyperparameters used for LLC estimation (see Table \ref{table:symbols} for a reminder). Recall that $w^*$ is the center of the localization term in the SGLD update; we assume $w^* \in W'$. Let $\{w_t\}_{t=1}^T$ be a SGLD chain contained in $W'$ with updates given in \Cref{eq: SGLD update}. Let $\{\tilde{w}_t\}_{t=1}^T$ be the same SGLD chain as $\{w_t\}_{t=1}^T$ except that $\nabla_w L_m$ is replaced by $\nabla_w L_m^{(\chi)}$. We assume the same random noise $\eta_t$ is used in both chains for simplicity.

\begin{theorem}\label{thm: SGLD theorem} Suppose that the model is gradient-insensitive in $W'$ to the difference between $q$ and $q^{(\chi)}$ for the constant $A$ and that $\nabla_w L_m(w)$ is Lipschitz continuous on $W'$ with Lipschitz constant $M$ for $m \gg 0$. Provided $Mn\beta \in \big(\gamma - \tfrac{2}{\epsilon_{\max}}, \gamma\big)$. Let $\xi$, $\delta > 0$, then for large enough $m$
\[
    \mathbb{P}\Big(\| w_t - \tilde{w}_t \|_2 < g(t, A) \Big) > 1 - f(t, \delta), \qquad \forall 1 \le t \le T
\]
where
\begin{enumerate}
    \item $g(t, A) =  \frac{\epsilon_{\max}}{\epsilon_{\min}}\frac{ (A + \xi)}{\frac{\gamma}{n\beta}-M}(1-\mu^{t-1})$, where $\mu = 1 + \frac{\epsilon_{\min}}{2}( M n \beta - \gamma) \in (0,1)$;
    \item $f(t, \delta) = t\delta$.
\end{enumerate}
\end{theorem}
\begin{proof}
Let $\{ a_t \}_{t \ge 1}, \{ b_t \}_{t \ge 1}$ be non-negative and $\mu \in [0,1)$. If for $t \ge 1$
\[
a_{t+1} \le b_t + \mu a_t
\]
then with $a_1$ given
\[
a_{t+1} \le b_t + \mu b_{t - 1} + \mu^2 b_{t - 2} + \cdots + \mu^{t-1} b_1 + \mu^{t} a_1
\]
and, in particular, when $b_t = b$ is a constant and $a_1 = 0$,
\begin{equation}\label{eq:gronwald}
a_{t+1} \le b (1 + \mu + \mu^2 + \cdots + \mu^{t-1}) + \mu^{t} a_1 = \frac{b}{1 - \mu} (1 - \mu^t).
\end{equation}
Set $\Delta_t = \| w_t - \tilde{w}_t \|_2$ and let $\epsilon_t$ be the step-size at time $t$. Let $A$ be as in Lemma \ref{lemma: empirical version of negligible modes} so that
\begin{equation}\label{eq:first_theorem_Abound}
\big\| \nabla_w L_m(w) - \nabla_w L_m^{(\chi)}(w) \big\|_2 < A + \xi
\end{equation}
with probability $\ge 1 - \delta$ for all $w \in W'$. From the update rule
\begin{align*}
\Delta_{t+1} &= \| w_{t+1} - \tilde{w}_{t+1} \|_2\\
&= \| w_{t+1} - w_t - \tilde{w}_{t+1} + \tilde{w}_t + (w_t - \tilde{w}_t) \|_2\\
&= \| \Delta w_t - \Delta \tilde{w}_t + (w_t - \tilde{w}_t) \|_2\\
&= \Big\| \frac{\epsilon_t}{2}\Big[ -\beta n( \nabla_w L_m(w_t) - \nabla_w L_m^{(\chi)}(\tilde{w}_t) ) + \gamma(-w_t + \tilde{w}_t) \Big] + (w_t - \tilde{w}_t) \Big\|_2\\
&\le \epsilon_t \beta \frac{n}{2} \Big\| \nabla_w L_m(w_t) - \nabla L_m^{(\chi)}(\tilde{w}_t) \Big\|_2 + \big(1 - \epsilon_t \frac{\gamma}{2}\big) \Delta_t\,.
\end{align*}
Now conditioned on an event with probability $\ge 1 - \delta$ and using the Lipschitz constant $M$ for $\nabla_w L_m$ we can apply \eqref{eq:first_theorem_Abound} with $w = \tilde{w}_t$
\begin{align*}
\big\| \nabla_w L_m(w_t) - \nabla L_m^{(\chi)}(\tilde{w}_t) \big\|_2 &\le \big\| \nabla_w L_m(\tilde{w}_t) - \nabla_w L^{(\chi)}_m(\tilde{w}_t)\big\|_2\\
& \qquad + \big\| \nabla_w L_m(w_t) - \nabla_w L_m(\tilde{w}_t)\big\|_2\\
&\le (A + \xi) + M \Delta_t\,.
\end{align*}
Hence under the same event
\begin{align*}
\Delta_{t+1} &\le \epsilon_t \beta \frac{n}{2}\Big( (A + \xi) + M \Delta_t \Big) + \big(1 - \epsilon_t \frac{\gamma}{2}\big) \Delta_t\\
&= \frac{n\beta\epsilon_t (A + \xi)}{2} + \Big( 1 + \frac{\epsilon_t}{2}( M n \beta - \gamma) \Big) \Delta_t\,.
\end{align*}
We have $1 + \frac{\epsilon_t}{2}( M n \beta - \gamma) \in (0,1)$ if and only if $Mn\beta \in \big(\gamma - \tfrac{2}{\epsilon_t}, \gamma\big)$. By our assumption, 
\[
    Mn\beta \in \big(\gamma - \tfrac{2}{\epsilon_{\max}}, \gamma\big) \subset (\gamma - \tfrac{2}{\epsilon_t}, \gamma\big) \,.
\]
Note that $Mn\beta - \gamma < 0$ so
\begin{align*}
     \Delta_{t+1} &\le  \frac{n\beta\epsilon_{\max} (A+\xi)}{2} + \big( 1 + \frac{\epsilon_{\min}}{2}( M n \beta - \gamma) \big) \Delta_t\,.
\end{align*}
Set
\[
    \mu = 1 + \frac{\epsilon_{\min}}{2}(Mn\beta - \gamma) \in (0, 1)\,.
\]
We apply \eqref{eq:gronwald}:
\begin{align}
\Delta_{t+1} &\le \frac{n\beta \epsilon_{\max} (A + \xi)}{2} \frac{1}{1-\mu}(1 - \mu^t)\\
&= \frac{n \beta \epsilon_{\max} (A+\xi)}{\epsilon_{\min}( \gamma - Mn\beta)} ( 1- \mu^t)\\
&= \frac{\epsilon_{\max}}{\epsilon_{\min}}\frac{ (A + \xi)}{\frac{\gamma}{n\beta} - M}(1-\mu^t) \label{eq:deltatplus1}
\end{align}
Hence set
\[
    g(t, A) =  \frac{\epsilon_{\max}}{\epsilon_{\min}}\frac{ A + \xi}{\frac{\gamma}{n\beta}-M}(1-\mu^{t-1})\,.
\]
Now let $E_t$ be the event
\[
E_t = \big\{ \big\| \nabla L_m(w_\tau) - \nabla L_m^{(\chi)}(w_\tau) \big\|_2 \le A+\xi \, \text{ for all } \tau = 1, \ldots, t \big\}
\]
By a union bound $\mathbb{P}(E_t) \ge 1 - t\delta$. On $E_t$ the inequality \eqref{eq:deltatplus1} holds for $\tau = 1,\ldots,t$, hence
\[
\mathbb{P}( \Delta_t \le g(t,A) ) \ge 1 - t\delta
\]
as claimed with $f(t,\delta) = t\delta$.
\end{proof}

Let $w_t$ be a SGLD chain of length $T$ with initial point $w_1$. Let $\tilde{w}_t$ be another SGLD chain of the same length $T$ with the same initial point $\tilde{w}_1 = w_1$ and the same injected noise:
\begin{align*}
    \Delta w_t &= \frac{\epsilon_t}{2}\Big[-\beta n \nabla_w L_m(w_t) + \gamma(w^* - w_t)\Big] + \eta_t,\\
    \Delta \tilde{w}_t &= \frac{\epsilon_t}{2}\Big[-\beta n \nabla_w L^{(\chi)}_m(\tilde{w}_t) + \gamma(w^* - \tilde{w}_t)\Big] + \eta_t\,.
\end{align*}
Let $\hat{\lambda}_n(\wstar)$ denote the estimated LLC (\ref{eq: SGLD based LLC estimator}) using $L_n(w)$ and samples $\{w_t\}_{t=1}^T$, and $\hat{\lambda}_n^{(\chi)}(\wstar)$ denote the estimated LLC using $L_n^{(\chi)}(w)$ and samples $\{\tilde{w}_t\}_{t=1}^T$.

\begin{table}[hptb]
\centering
\begin{tabular}{|c|p{7cm}|c|}
\hline
\textbf{Symbol} & \textbf{Description} & \textbf{Reference} \\
\hline
$A$ & Constant for gradient-insensitivity & Definition \ref{defn:gradient_insensitive} \\
$B$ & Constant for log-probability-insensitivity & Definition \ref{defn:insensitive} \\
\hline
$\beta$ & Inverse temperature in SGLD & Section \ref{section:llc_estimation} \\
$\epsilon_t$ & Step-size in SGLD & Section \ref{section:llc_estimation} 
\\
$\epsilon_{\max}$, $\epsilon_{\min}$ & Maximum and minimum step-size in SGLD & before Theorem \ref{thm: SGLD theorem}\\
$n$ & Sample size used to define $L_n$ & Section \ref{section:llc_estimation} \\
$m$ & Size of mini-batches in SGLD & Section \ref{section:llc_estimation} \\
$\gamma$ & Localisation strength in SGLD & Section \ref{section:llc_estimation} \\
$T$ & Length of SGLD chains & Section \ref{section:llc_estimation} \\
\hline
$M$ & Lipschitz constant for $\nabla_w L_m$ & Theorem \ref{theorem:main_llc} \\
$Q$ & Lipschitz constant for $L_n$ & Theorem \ref{theorem:main_llc} \\
\hline
\end{tabular}
\vspace{0.5cm}
\caption{Table of Notation.}
\label{table:symbols}
\end{table}

\begin{theorem}\label{theorem:main_llc}
Suppose that the model is insensitive in a compact set $W'$ to the difference between $q$ and $q^{(\chi)}$ in both senses:
\begin{itemize}
    \item It is gradient-insensitive for a constant $A$
    \item It is log-probability-insensitive for a constant $B$.
\end{itemize}
Assume that $\nabla_w L_m(w)$ and $L_n(w)$ are Lipschitz continuous on $W'$ with Lipschitz constant $M$ and $Q$, respectively, for $m \gg 0$ and $n \gg 0$ and that $Mn\beta \in \big(\gamma - \tfrac{2}{\epsilon_{\max}}, \gamma\big)$. Then for any $\delta$, $\xi$, $\kappa > 0$, there exists large enough $n$ such that
\begin{equation}
    \mathbb{P}\bigg(\big|\hat{\lambda}_n(\wstar) - \hat{\lambda}_n^{(\chi)}(\wstar)\big| < \frac{\epsilon_{\max}}{\epsilon_{\min}} \frac{ n\beta Q (A+\xi)}{\frac{\gamma}{n\beta}-M} + 2n\beta(B+ \kappa) \bigg) > 1  - \tfrac{1}{2} T (T+1) \delta - \kappa \,.
\end{equation}
\end{theorem}
\begin{proof}
Applying the triangle inequality, we have 
\begin{align}
    \big|\hat{\lambda}_n(w^*) - \hat{\lambda}_n^{(\chi)}(w^*) \big| &= \bigg|n\beta \frac{1}{T} \sum_{t=1}^T \big(L_n(w_t) - L_n^{(\chi)}(\tilde{w}_t)\big) + n\beta\big(L_n^{(\chi)}(\wstar) - L_n(\wstar)\big) \bigg| \nonumber \\
    &\leq n\beta \frac{1}{T} \sum_{t=1}^T \big|L_n(w_t) - L_n^{(\chi)}(\tilde{w}_t)\big| + n\beta \big| L_n^{(\chi)}(\wstar) - L_n(\wstar) \big| \nonumber \\
    &\leq n\beta\frac{1}{T} \sum_{t=1}^T \big| L_n(w_t) - L_n(\tilde{w}_t) \big| + n\beta\frac{1}{T} \sum_{t=1}^T \big| L_n(\tilde{w}_t) - L_n^{(\chi)}(\tilde{w}_t)\big|  \nonumber \\
    &\qquad + n\beta\big| L_n(w^*) - L_n^{(\chi)}(w^*) \big| \,. \nonumber
\end{align}
Since $L_n(w)$ is Lipschitz continuous with Lipschitz constant $Q$, we have 
\begin{equation*}
    n\beta \frac{1}{T} \sum_{t=1}^T \big|L_n(w_t) - L_n(\tilde{w}_t)\big| \leq C \cdot \frac{1}{T} \sum_{t=1}^T \|w_t - \tilde{w}_t\| \,,
\end{equation*}
where $C = n\beta Q$. So
\begin{align}
    \big|\hat{\lambda}_n(w^*) - \hat{\lambda}_n^{(\chi)}(w^*) \big| &\leq C \cdot \frac{1}{T} \sum_{t=1}^T \| w_t - \tilde{w}_t \| + n\beta\frac{1}{T} \sum_{t=1}^T \big| L_n(\tilde{w}_t) - L_n^{(\chi)}(\tilde{w}_t)\big|  \nonumber \\
    &\qquad + n\beta\big| L_n(w^*) - L_n^{(\chi)}(w^*) \big| \,. \nonumber
\end{align}
Thus, for $g(t, A)$ defined in \Cref{thm: SGLD theorem}, we have by a union bound
\begin{align}
    & \quad \mathbb{P}\bigg(\big|\hat{\lambda}_n(\wstar) - \hat{\lambda}_n^{(\chi)}(\wstar)\big| < C \cdot \frac{1}{T} \sum_{t=1}^T g(t, A) + 2n\beta (B + \kappa) \bigg) \nonumber \\
    &\geq \mathbb{P} \bigg( C \cdot \frac{1}{T} \sum_{t=1}^T \| w_t - \tilde{w}_t \| + n\beta\frac{1}{T} \sum_{t=1}^T \big| L_n(\tilde{w}_t) - L_n^{(\chi)}(\tilde{w}_t)\big|  \nonumber \\
    &\qquad + n\beta\big| L_n(w^*) - L_n^{(\chi)}(w^*) \big| < C \cdot \frac{1}{T} \sum_{t=1}^T g(t, A) + 2n\beta (B + \kappa) \bigg)   \nonumber \\
    &\geq \mathbb{P}\bigg( C \cdot \frac{1}{T} \sum_{t=1}^T \| w_t - \tilde{w}_t \|  < C \cdot \frac{1}{T}\sum_{t=1}^T g(t, A) \bigg) \nonumber \\
    & \qquad + \mathbb{P}\bigg( n\beta \frac{1}{T}\sum_{t=1}^T \big| L_n(\tilde{w}_t) - L_n^{(\chi)}(\tilde{w_t}) \big| < n\beta(B + \kappa) \bigg) \nonumber \\
    & \qquad + \mathbb{P}\bigg( n\beta \big| L_n(w^*) - L_n^{(\chi)}(w^*) \big| < n\beta(B + \kappa) \bigg) + 1 - 3 \label{eq: step1}.
\end{align}
Hence
\begin{align}
    \mathbb{P}\bigg( C \cdot \frac{1}{T} \sum_{t=1}^T \|w_t - \tilde{w}_t\| < C \cdot \frac{1}{T} \sum_{t=1}^T g(t, A) \bigg)
    &= \mathbb{P}\bigg(\sum_{t=1}^T \|w_t - \tilde{w}_t\| < \sum_{t=1}^T g(t, A)\bigg) \nonumber \\
    &\geq \sum_{t=1}^T \mathbb{P}\big(\|w_t - \tilde{w}_t\| < g(t, A) \big) + 1 - T \,. \nonumber
\end{align}
Applying Theorem \ref{thm: SGLD theorem}, we have
\begin{align}
    \sum_{t=1}^T \mathbb{P}\big(\|w_t - \tilde{w}_t\| < g(t, A) \big) + 1 - T &> 1 - T + \sum_{t=1}^T \big(1 - f(t, \delta)\big) \nonumber \\
    &= 1 - \sum_{t=1}^T f(t, \delta) \,. \nonumber
\end{align}
Also
\begin{align*}
    &\quad \mathbb{P}\bigg(n\beta\frac{1}{T}\sum_{t=1}^T \big| L_n(\tilde{w}_t) - L_n^{(\chi)}(\tilde{w_t}) \big| < n\beta (B + \kappa) \bigg) \\
    &= \mathbb{P}\bigg( \sum_{t=1}^T \big| L_n(\tilde{w}_t) - L_n^{(\chi)}(\tilde{w_t}) \big| < T(B + \kappa) \bigg) \\
    &\geq \sum_{t=1}^T \mathbb{P}\big( \big| L_n(\tilde{w}_t) - L_n^{(\chi)}(\tilde{w}_t) \big| < B + \kappa \big) + 1 - T .
\end{align*}
We know by insensitivity that for all $w \in W'$, $\big| L(w) - L^{(\chi)}(w) \big| < B$. Since $L_n(w) \to L(w)$ and $L_n^{(\chi)}(w) \to L^{(\chi)}(w)$ for all $w \in W'$ uniformly as $n \to \infty$, there exists large enough $n$ such that 
\begin{equation*}
    \mathbb{P}\big( \big| L_n(\tilde{w}_t) - L_n^{(\chi)}(\tilde{w}_t) \big| < B+\kappa \big) > 1 - \frac{\kappa}{2T} \,.
\end{equation*}
Thus, 
\begin{align*}
    \mathbb{P}\bigg(n\beta\frac{1}{T}\sum_{t=1}^T \big| L_n(\tilde{w}_t) - L_n^{(\chi)}(\tilde{w_t}) \big| < n\beta (B + \kappa) \bigg)& > \sum_{t=1}^T \bigg(1 - \frac{\kappa}{2T} \bigg) + 1 - T \\
    &= 1 - \frac{\kappa}{2} \,.
\end{align*}
We also have 
\begin{align*}
    \mathbb{P}\bigg( n\beta \big| L_n(w^*) - L_n^{(\chi)}(w^*) \big| < n\beta(B + \kappa) \bigg) &= \mathbb{P}\big( \big| L_n(w^*) - L_n^{(\chi)}(w^*) \big| < B + \kappa \big).
\end{align*}
Similarly, there exists large enough $n$ such that
\begin{align*}
    \mathbb{P}\big( \big| L_n(w^*) - L_n^{(\chi)}(w^*) \big| < B + \kappa \big) > 1 - \frac{\kappa}{2} \,.
\end{align*}
It follows from \Cref{eq: step1} that for any $\kappa > 0$, there exists large enough $n$ such that
\begin{align*}
    &\quad \mathbb{P}\bigg(\big|\hat{\lambda}_n(\wstar) - \hat{\lambda}_n^{(\chi)}(\wstar)\big| < C \cdot \frac{1}{T} \sum_{t=1}^T g(t, A) + 2n\beta (B + \kappa) \bigg) \\
    &\geq \bigg(1 - \sum_{t=1}^T f(t, \delta) \bigg) + \bigg(1 - \frac{\kappa}{2}\bigg) + \bigg(1 - \frac{\kappa}{2}\bigg) + 1 -3 \\
    &= 1 - \sum_{t=1}^T f(t, \delta) - \kappa \,.
\end{align*}
Using the $g(t,A), f(t,\delta)$ from Theorem \ref{thm: SGLD theorem} we have
\begin{align*}
\frac{1}{T}\sum_{t=1}^T g(t,A) &= \frac{1}{T} \sum_{t=1}^T \frac{\epsilon_{\max}}{\epsilon_{\min}}  \frac{ A + \xi}{\frac{\gamma}{n\beta}-M}(1-\mu^{t-1}) = \frac{\epsilon_{\max}}{\epsilon_{\min}} \frac{ A + \xi}{\frac{\gamma}{n\beta}-M} \Big( 1 - \frac{1}{T} \sum_{t=1}^T \mu^{t-1} \Big)\\
&= \frac{\epsilon_{\max}}{\epsilon_{\min}}  \frac{ A + \xi}{\frac{\gamma}{n\beta}-M}\Big( 1 - \frac{1}{T} \frac{1-\mu^{T}}{1-\mu} \Big) \le \frac{\epsilon_{\max}}{\epsilon_{\min}} 
 \frac{ A + \xi}{\frac{\gamma}{n\beta}-M}
\end{align*}
except possibly for very small $T$. And
\begin{align*}
\sum_{t=1}^T f(t,\delta) = \sum_{t=1}^T t\delta = \tfrac{1}{2}T(T+1) \delta
\end{align*}
Hence
\begin{equation*}
    \mathbb{P}\bigg(\big|\hat{\lambda}_n(\wstar) - \hat{\lambda}_n^{(\chi)}(\wstar)\big| < \frac{\epsilon_{\max}}{\epsilon_{\min}} \frac{ n\beta Q (A+\xi)}{\frac{\gamma}{n\beta}-M} + 2n\beta(B+ \kappa) \bigg) > 1  - \tfrac{1}{2} T (T+1) \delta - \kappa \,.
\end{equation*}
\end{proof}

\begin{remark}\label{remark:practical_bound} 
Let us now study the hypotheses and the bound (we set $\xi = \kappa = 0$)
\begin{equation}\label{eq:prac_bound}
\frac{\epsilon_{\max}}{\epsilon_{\min}} \frac{ n\beta Q A}{\frac{\gamma}{n\beta}-M} + 2n\beta B
\end{equation}
for hyperparameters used in practice for LLC estimation. For recent experiments on Pythia-6.9B \citep{biderman2023pythia} we use
\[
\epsilon_{\text{min}} = \epsilon_{\text{max}} = \epsilon \in \{10^{-4}, 10^{-5}\}\,, n \beta = 10\,, \gamma = 300\,, T = 100\,.
\]
The first question is how realistic is it to ask $Mn\beta \in \big(\gamma - \tfrac{2}{\epsilon_{\max}}, \gamma\big)$. With the above hyperparameters this condition is $M < 30$. Any bound on the spectral norm $\Vert \nabla^2 L_m \Vert_2$ of the Hessian of $L_m$ in $W'$ gives a Lipschitz constant $M$. This would follow from a bound on the largest Hessian eigenvalue inside $W'$. In \citep{yao2020pyhessianneuralnetworkslens} a trained ResNet (270k parameters) at the end of training had estimated maximum Hessian eigenvalue $150$. We were unable to find empirical estimates of Hessian spectra for large transformers. It would be feasible to run LLC estimates with $n\beta$ decreased by some small factor and $\gamma$ increased by a small factor (say 2X in both cases, and likely with modified step-sizes to accommodate the decrease in $n\beta$) which would put the bound at the same order of magnitude as these estimates of Hessian eigenvalues for ResNets; however at present it would not be realistic if the Lipschitz constant were orders of magnitude larger.

Next we discuss the Lipschitz constant $Q$. This can be bounded by the norm of the gradient of $L_n$. In a recent training run of 7B and 70B LLMs, \citet{fan2024hlat} found gradient norms $< 10$ except for rare spikes. So $Q < 10$ seems reasonable.
\end{remark}

\subsection{Applications}\label{section:mixing_lengths}

The training objective of a transformer language model is for $p(x|w)$ to model the distribution of sequences $x \in \Sigma^K$ where $k$ is the context length; see \eqref{eq:training_loss_transformer_normal}. However given a \emph{decomposition} of $K$ in the following sense
\begin{itemize}
    \item  $k_i, l_i \in \mathbb{Z}_{>0}$ for $i = 1, \dots, m$;
    \item $k_1 + l_1 = K$;
    \item $k_i = k_{i+1} + l_{i+1}$ for $i = 1, \dots, m-1$
\end{itemize}
we can think of this as a sequence of models (writing $x_{a:b} = x_a \cdots x_{b}$ and $x_{>b} = x_{b+1} \cdots x_K$)
\begin{align*}
p(x|w) &= p(x_{> k_1}|x_{1:k_1},w) p(x_{1:k_1}|w)\\
&= p(x_{> k_1}|x_{1:k_1},w) p(x_{k_2+1:k_1}|x_{1:k_2},w) p(x_{1:k_2}|w)\\
&\vdots
\end{align*}
That is, we can think of a transformer parameter $w \in W$ as jointly parametrising models $p(y|x,w)$ for $x \in \Sigma^k$, $y \in \Sigma^l$ for $(k,l) = (k_i, l_i)$ and decompose the overall population loss $L(w)$ as a sum of losses $L_{k_i,l_i}(w)$. We can choose mode cutoffs $\chi_i$ for each $1 \le i \le m - 1$ and impose log-probability-insensitivity and gradient-insensitivity conditions for each $i$. Then by adapting the proofs of the theorems above we obtain a bound on the difference of the LLC of $L$ and $L^{(\overline{\chi})}$ where $\overline{\chi} = (\chi^1,\ldots,\chi^{m-1})$ is obtained by truncating each of the distributions (see Appendix \ref{section:mult_seq_lengths}).

Further, we may use higher-order versions of SVD to decompose the fundamental tensors involved in determining these sets of modes (see Appendix \ref{section:hod_tensor}).

\section{Discussion}\label{section:discussion}

In the previous section we showed, under some technical assumptions, that the geometry of the functions $L, L^{(\chi)}$ cannot be distinguished by LLC estimation (Theorem \ref{theorem:main_llc}). In this sense $L^{(\chi)}$ acts as an \emph{effective potential}. In this section we make some general remarks on what this means for the practice of studying language models like transformers using LLC estimates.

\paragraph{Local Minima.} The technical conditions for LLC estimation \citep{quantifdegen} include that $w^*$ is a local minima of the population loss $L$. However do not expect this to be typically the case for network parameters $w^*$ that we wish to study in deep learning; in particular, it is unlikely to be true for language models like those studied in \citet{hoogland2024developmental}, \citet{wang2024differentiationspecializationattentionheads}. As observed in \citep[\S A.6]{hoogland2024developmental} this makes it quite surprising that LLC estimation works at all; it was speculated there that ``the LLC estimation procedure is effectively giving us the LLC of a different potential to the population loss -- a potential for which the current parameter actually is at a critical point.'' We do not completely resolve this point in the present paper, but if $L^{(\chi)}$ represents the loss on all the modes that the model has learned, it seems certainly \emph{more} reasonable that a parameter $w^*$ found during training may be a local minima for this potential rather than $L$.

\paragraph{Inverse Temperature versus Mode Cutoff.} Suppose we have chosen values for all these hyperparameters, including $A,B$ such that the hypotheses of Theorem \ref{theorem:main_llc} are satisfied and give a non-trivial bound on the difference between $\hat{\lambda}_n(\wstar)$ and $\hat{\lambda}_n^{(\chi)}(\wstar)$. Suppose that this bound is small enough that it means ``in practice'' we can't distinguish the LLC of $q$ and $q^{(\chi)}$ with the resolution of our probing of the posterior via SGLD. Note that the inverse temperature $\beta$ is a common factor in the bound \eqref{eq:prac_bound}. This means that if we increase $\beta$ but decrease $A,B$ to match, we can maintain the same level of in-distinguishability. Decreasing $A,B$ means intuitively that we need to shift the mode cutoff $\chi$ higher, so that the insensitivity conditions become easier to satisfy (as $q - q^{(\chi)}$ is smaller). Similarly if we decrease $\beta$ then $A,B$ can be allowed to increase, which relaxes the constraints and allows the cutoff to shift down (that is, our estimates capture coarser modes).

In practice what this means is that the inverse temperature at which we attempt to sample from the posterior with SGLD plays an important role in controlling the level of resolution in the geometry that we measure; here $L^{(\chi)}$ is thought of as a higher resolution approximation of $L$ as $\chi$ increases. For example, the LLC estimate will count a contribution to the complexity of the model of rarer bigrams as $\beta$ is increased since their singular value is the square root of their probability.

In the empirical literature on LLC estimation for transformer neural networks \citep{hoogland2024developmental,wang2024differentiationspecializationattentionheads,urdshals2025,carroll2025dynamicstransientstructureincontext} the effect of variations in $\beta$ has not yet been well-studied, beyond hyperparameter selection. The potential importance of variations of inverse temperature for neural network interpretability have been emphasised by \citet{vaintrob}. More generally, we can think of variations in $\beta$ as one of a number of perturbations we can make to the (tempered) posterior distribution viewed as an object of statistical mechanics. We can also perform variations in the data distribution itself. We read the main theorem of this paper as suggesting that such variations affect the LLC estimate \emph{through the effect of these variations on individual modes}. The study of changes in LLC estimates in response to such variations as a foundation for interpretability is studied under the name \emph{susceptibilities} in upcoming work by \citet{suscep}.

\section{Related Works}

\paragraph{Distributional structure and deep learning.} The study of distributional structure in language has a rich history in linguistics and computational linguistics. \citet{harris1954distributional} pioneered the idea that the meaning of words could be inferred from their distribution in text, laying the groundwork for modern approaches to semantic analysis. This distributional hypothesis was further developed by \citet{firth1957synopsis}. In recent years, the advent of neural network-based approaches has revolutionised how we capture and utilize distributional information in language. \citet{bengio2003neural} introduced neural probabilistic language models, which implicitly learn distributional representations of words. This work paved the way for more sophisticated models that explicitly aim to capture distributional semantics \citep{mikolov2013distributed,pennington2014glove}.

\paragraph{Tensor decomposition in NLP.} In the context of natural language, tensor decompositions have been widely used both theoretically and for practical applications. For our purposes the canonical references are \citet{anandkumar2014tensor} and \citet{hsu2012spectral}.

\paragraph{Dimensionality Reduction for Markov Models.} Traditional approaches to approximating or coarse-graining Markov chains have utilized spectral methods to identify essential dynamics \citep{deuflhard2005robust,deuflhard2000identification,fritzsche2007svd}. \citet{geiger2014optimal} frames state aggregation as an information bottleneck problem, seeking to compress the current state while preserving information about future trajectories. There is an extensive literature on spectral methods for ``compressing'' Markov processes using tensor decompositions \citep{zhang2019spectral,ohta2021realization,navarro2024low} with the underlying hypothesis being that complex systems can often be lower rank than they appear \citep{thibeault2024low}. The relation of these ideas to renormalisation group theory has also been recognised \citep{orioli2016dimensional}.

\paragraph{Tensor Renormalisation Group.} Our approach to mode truncation in sequence distributions shares conceptual foundations with tensor renormalisation group (TRG) methods in statistical physics \citep{levin2007tensor}. TRG provides a framework for coarse-graining tensor networks by iteratively replacing tensors with their truncated SVDs. The truncation procedure defined in Definition \ref{defn: truncated distribution}, where we construct $q^{(\chi)}(y|x)$ by retaining only modes $\alpha \leq \chi$, directly corresponds to the truncated SVD operation central to TRG. In both cases, truncation serves to control computational complexity while preserving dominant patterns. Notably, research has shown that the ``brute force'' truncation of singular values in the original TRG approach, while effective, is not optimal for capturing the relevant physics near critical points \citep{evenbly2015tensor}. Some ideas from the TRG literature (such as hierarchical truncation) may inspire useful approaches for improved LLC estimation techniques.

\paragraph{ Effective Actions in MCMC.} In lattice field theory, it is well-established that Langevin dynamics with finite step size $\varepsilon$ samples from a modified distribution rather than the target distribution \citep{batrouni1985langevin}. Specifically, when attempting to sample from a distribution $P[\phi] \propto e^{-S[\phi]}$ using the update rule:
\[
\Delta \phi(x) = -\varepsilon \frac{\delta S[\phi]}{\delta \phi(x)} + \sqrt{\varepsilon} \eta(x)
\]
the algorithm actually samples from a distribution with an \textit{effective action} $S_{\text{eff}}[\phi]$ which is analagous to our effective potential. The techniques used there suggest trying to approach the effective distribution $q^{(\chi)}$ directly through analysis of the Fokker-Planck equation associated with SGLD. Techniques from physics for managing effective action corrections might enable more accurate LLC estimation across a broader spectrum of modes.

This connection also reframes our view of the ``effective theory'' presented in this paper. Rather than seeing mode truncation as an approximation, we can understand it as identifying the appropriate level of description given the computational procedure used for LLC estimation. This is analogous to how renormalisation group methods identify effective field theories valid at specific length scales \citep{polchinski1984renormalization,efrati2014real}.

\section{Conclusion}

We have shown that the geometry probed using SGLD for LLC estimation in sequence models is, in practice, the geometry of a \emph{coarse-grained} or \emph{effective} distribution obtained by truncating higher‐order modes of the fundamental tensors. More precisely, using a mode decomposition based on singular-value analysis we proved that, under conditions of gradient- and log-probability-insensitivity, SGLD-based LLC estimates cannot distinguish the true potential $L$ from the truncated potential $L^{(\chi)}$.

This perspective may help clarify longstanding puzzles about why LLC estimates appear stable even when the measured parameter is not a strict minimum of the population loss: the estimator may be effectively sampling a potential where the learned modes are stationary. The perspective of modes introduced here, and the relation of modes to LLC estimation, is also fundamental in our approach to neural network interpretability based on SLT \citep{suscep}.

Future work should explore alternative tensor factorizations (such as CP, Tucker and tensor-train) to capture more complex hierarchies of modes, tighten quantitative bounds on the insensitivity constants for modern transformer architectures and develop a more sophisticated treatment of the main theorems that allows keeping $m$ small, which is more consistent with how SGLD is used in practice.

\bibliography{references}

\begin{thebibliography}{68}
\providecommand{\natexlab}[1]{#1}
\providecommand{\url}[1]{\texttt{#1}}
\expandafter\ifx\csname urlstyle\endcsname\relax
  \providecommand{\doi}[1]{doi: #1}\else
  \providecommand{\doi}{doi: \begingroup \urlstyle{rm}\Url}\fi

\bibitem[Amari(2009)]{Amari2009}
Amari, S.-i.
\newblock \emph{Information Geometry and Its Applications: Convex Function and Dually Flat Manifold}, pp.\  75--102.
\newblock Springer Berlin Heidelberg, 2009.

\bibitem[Amari(2016)]{amari2016information}
Amari, S.-i.
\newblock \emph{Information Geometry and its Applications}.
\newblock Springer, 2016.

\bibitem[Anandkumar et~al.(2012)Anandkumar, Ge, Hsu, Kakade, and Telgarsky]{anandkumar2012method}
Anandkumar, A., Ge, R., Hsu, D., Kakade, S.~M., and Telgarsky, M.
\newblock A method of moments for mixture models and hidden {M}arkov models.
\newblock \emph{Conference on Learning Theory}, pp.\  33--1, 2012.

\bibitem[Anandkumar et~al.(2014)Anandkumar, Ge, Hsu, Kakade, Telgarsky, et~al.]{anandkumar2014tensor}
Anandkumar, A., Ge, R., Hsu, D.~J., Kakade, S.~M., Telgarsky, M., et~al.
\newblock Tensor decompositions for learning latent variable models.
\newblock \emph{J. Mach. Learn. Res.}, 15\penalty0 (1):\penalty0 2773--2832, 2014.

\bibitem[Aoyagi(2024)]{AOYAGI2024106132}
Aoyagi, M.
\newblock Consideration on the learning efficiency of multiple-layered neural networks with linear units.
\newblock \emph{Neural Networks}, 172:\penalty0 106132, 2024.
\newblock ISSN 0893-6080.

\bibitem[Baker et~al.(2025)Baker, Wang, Hoogland, and Murfet]{suscep}
Baker, G., Wang, G., Hoogland, J., and Murfet, D.
\newblock Studying {S}mall {L}anguage {M}odels with {S}usceptibilities, 2025.

\bibitem[Batrouni et~al.(1985)Batrouni, Katz, Kronfeld, Lepage, Svetitsky, and Wilson]{batrouni1985langevin}
Batrouni, G., Katz, G., Kronfeld, A.~S., Lepage, G., Svetitsky, B., and Wilson, K.
\newblock Langevin simulations of lattice field theories.
\newblock \emph{Physical Review D}, 32\penalty0 (10):\penalty0 2736, 1985.

\bibitem[Bengio et~al.(2003)Bengio, Ducharme, Vincent, and Jauvin]{bengio2003neural}
Bengio, Y., Ducharme, R., Vincent, P., and Jauvin, C.
\newblock A neural probabilistic language model.
\newblock \emph{Journal of machine learning research}, 3\penalty0 (Feb):\penalty0 1137--1155, 2003.

\bibitem[Biderman et~al.(2023)Biderman, Schoelkopf, Anthony, Bradley, O’Brien, Hallahan, Khan, Purohit, Prashanth, Raff, et~al.]{biderman2023pythia}
Biderman, S., Schoelkopf, H., Anthony, Q.~G., Bradley, H., O’Brien, K., Hallahan, E., Khan, M.~A., Purohit, S., Prashanth, U.~S., Raff, E., et~al.
\newblock Pythia: A suite for analyzing large language models across training and scaling.
\newblock In \emph{International Conference on Machine Learning}, pp.\  2397--2430. PMLR, 2023.

\bibitem[Black et~al.(2022)Black, Biderman, Hallahan, Anthony, Gao, Golding, He, Leahy, McDonell, Phang, et~al.]{black2022gpt}
Black, S., Biderman, S., Hallahan, E., Anthony, Q., Gao, L., Golding, L., He, H., Leahy, C., McDonell, K., Phang, J., et~al.
\newblock Gpt-neox-20b: An open-source autoregressive language model.
\newblock \emph{arXiv preprint arXiv:2204.06745}, 2022.

\bibitem[Carroll \& Chang(1970)Carroll and Chang]{carroll1970analysis}
Carroll, J.~D. and Chang, J.-J.
\newblock Analysis of individual differences in multidimensional scaling via an {N}-way generalization of "{E}ckart-{Y}oung" decomposition.
\newblock \emph{Psychometrika}, 35\penalty0 (3):\penalty0 283--319, 1970.

\bibitem[Carroll et~al.(2025)Carroll, Hoogland, Farrugia-Roberts, and Murfet]{carroll2025dynamicstransientstructureincontext}
Carroll, L., Hoogland, J., Farrugia-Roberts, M., and Murfet, D.
\newblock Dynamics of transient structure in in-context linear regression transformers, 2025.
\newblock URL \url{https://arxiv.org/abs/2501.17745}.

\bibitem[Chen et~al.(2023)Chen, Lau, Mendel, Wei, and Murfet]{chen2023tms1}
Chen, Z., Lau, E., Mendel, J., Wei, S., and Murfet, D.
\newblock Dynamical versus {Bayesian} phase transitions in a toy model of superposition.
\newblock Preprint arXiv:2310.06301 [cs.LG], 2023.

\bibitem[Cover \& Thomas(2006)Cover and Thomas]{Thomas2006}
Cover, T.~M. and Thomas, J.~A.
\newblock \emph{Elements of Information Theory (Wiley Series in Telecommunications and Signal Processing)}.
\newblock Wiley-Interscience, USA, 2006.

\bibitem[Deuflhard \& Weber(2005)Deuflhard and Weber]{deuflhard2005robust}
Deuflhard, P. and Weber, M.
\newblock Robust {P}erron cluster analysis in conformation dynamics.
\newblock \emph{Linear algebra and its applications}, 398:\penalty0 161--184, 2005.

\bibitem[Deuflhard et~al.(2000)Deuflhard, Huisinga, Fischer, and Sch{\"u}tte]{deuflhard2000identification}
Deuflhard, P., Huisinga, W., Fischer, A., and Sch{\"u}tte, C.
\newblock Identification of almost invariant aggregates in reversible nearly uncoupled {M}arkov chains.
\newblock \emph{Linear Algebra and its Applications}, 315\penalty0 (1-3):\penalty0 39--59, 2000.

\bibitem[{Devin Gulliver}(2025)]{devingulliver2025monology}
{Devin Gulliver}.
\newblock Monology/pile-uncopyrighted.
\newblock https://huggingface.co/datasets/monology/pile-uncopyrighted, March 2025.

\bibitem[Efrati et~al.(2014)Efrati, Wang, Kolan, and Kadanoff]{efrati2014real}
Efrati, E., Wang, Z., Kolan, A., and Kadanoff, L.~P.
\newblock Real-space renormalization in statistical mechanics.
\newblock \emph{Reviews of Modern Physics}, 86\penalty0 (2):\penalty0 647--667, 2014.

\bibitem[Evenbly \& Vidal(2015)Evenbly and Vidal]{evenbly2015tensor}
Evenbly, G. and Vidal, G.
\newblock Tensor network renormalization.
\newblock \emph{Physical review letters}, 115\penalty0 (18):\penalty0 180405, 2015.

\bibitem[Fan et~al.(2024)Fan, Zhou, Huang, Raman, Fu, Gupta, Ram, Wang, and Huan]{fan2024hlat}
Fan, H., Zhou, H., Huang, G., Raman, P., Fu, X., Gupta, G., Ram, D., Wang, Y., and Huan, J.
\newblock {HLAT}: {H}igh-quality large language model pre-trained on {AWS} {T}rainium.
\newblock In \emph{2024 IEEE International Conference on Big Data (BigData)}, pp.\  2100--2109. IEEE, 2024.

\bibitem[Firth(1957)]{firth1957synopsis}
Firth, J.~R.
\newblock A synopsis of linguistic theory, 1930-1955.
\newblock In \emph{Studies in Linguistic Analysis}, pp.\  1--32. Philological Society, 1957.

\bibitem[Fritzsche et~al.(2007)Fritzsche, Mehrmann, Szyld, and Virnik]{fritzsche2007svd}
Fritzsche, D., Mehrmann, V., Szyld, D., and Virnik, E.
\newblock An {SVD} approach to identifying meta-stable states of {M}arkov chains.
\newblock \emph{Electron. Trans. Numer. Anal.}, 2007.

\bibitem[Gao et~al.(2020)Gao, Biderman, Black, Golding, Hoppe, Foster, Phang, He, Thite, Nabeshima, Presser, and Leahy]{gao2020pile}
Gao, L., Biderman, S., Black, S., Golding, L., Hoppe, T., Foster, C., Phang, J., He, H., Thite, A., Nabeshima, N., Presser, S., and Leahy, C.
\newblock The {Pile:} an {800GB} dataset of diverse text for language modeling.
\newblock Preprint arXiv:2101.00027 [cs.CL], 2020.

\bibitem[Geiger et~al.(2014)Geiger, Petrov, Kubin, and Koeppl]{geiger2014optimal}
Geiger, B.~C., Petrov, T., Kubin, G., and Koeppl, H.
\newblock Optimal {K}ullback--{L}eibler aggregation via information bottleneck.
\newblock \emph{IEEE Transactions on Automatic Control}, 60\penalty0 (4):\penalty0 1010--1022, 2014.

\bibitem[Harris(1954)]{harris1954distributional}
Harris, Z.~S.
\newblock Distributional structure.
\newblock \emph{Word}, 10\penalty0 (2-3):\penalty0 146--162, 1954.

\bibitem[Harshman(1970)]{harshman1970foundations}
Harshman, R.~A.
\newblock Foundations of the parafac procedure: Models and conditions for an "explanatory" multimodal factor analysis.
\newblock \emph{UCLA Working Papers in Phonetics}, 16:\penalty0 1--84, 1970.

\bibitem[Hoogland et~al.(2024)Hoogland, Wang, Farrugia-Roberts, Carroll, Wei, and Murfet]{hoogland2024developmental}
Hoogland, J., Wang, G., Farrugia-Roberts, M., Carroll, L., Wei, S., and Murfet, D.
\newblock The developmental landscape of in-context learning, 2024.
\newblock URL \url{https://arxiv.org/abs/2402.02364}.

\bibitem[Hsu et~al.(2012)Hsu, Kakade, and Zhang]{hsu2012spectral}
Hsu, D., Kakade, S.~M., and Zhang, T.
\newblock A spectral algorithm for learning hidden markov models.
\newblock \emph{Journal of Computer and System Sciences}, 78\penalty0 (5):\penalty0 1460--1480, 2012.

\bibitem[Hutter(2005)]{hutter2005universal}
Hutter, M.
\newblock \emph{Universal artificial intelligence: Sequential decisions based on algorithmic probability}.
\newblock Springer Science \& Business Media, 2005.

\bibitem[Kim et~al.(2016)Kim, Jernite, Sontag, and Rush]{kim2016character}
Kim, Y., Jernite, Y., Sontag, D., and Rush, A.~M.
\newblock Character-aware neural language models.
\newblock In \emph{Proceedings of the AAAI Conference on Artificial Intelligence}, volume~30, 2016.

\bibitem[Kolda \& Bader(2009)Kolda and Bader]{kolda2009tensor}
Kolda, T.~G. and Bader, B.~W.
\newblock Tensor decompositions and applications.
\newblock \emph{SIAM review}, 51\penalty0 (3):\penalty0 455--500, 2009.

\bibitem[Kudo \& Richardson(2018)Kudo and Richardson]{kudo2018sentencepiece}
Kudo, T. and Richardson, J.
\newblock Sentencepiece: A simple and language independent subword tokenizer and detokenizer for neural text processing.
\newblock \emph{arXiv preprint arXiv:1808.06226}, 2018.

\bibitem[Lau et~al.(2024)Lau, Furman, Wang, Murfet, and Wei]{quantifdegen}
Lau, E., Furman, Z., Wang, G., Murfet, D., and Wei, S.
\newblock The local learning coefficient: A singularity-aware complexity measure, 2024.
\newblock URL \url{https://arxiv.org/abs/2308.12108}.

\bibitem[Lehalleur et~al.(2025)Lehalleur, Hoogland, Farrugia-Roberts, Wei, Oldenziel, Wang, Carroll, and Murfet]{lehalleur2025eataialignment}
Lehalleur, S.~P., Hoogland, J., Farrugia-Roberts, M., Wei, S., Oldenziel, A.~G., Wang, G., Carroll, L., and Murfet, D.
\newblock You are what you eat -- ai alignment requires understanding how data shapes structure and generalisation, 2025.
\newblock URL \url{https://arxiv.org/abs/2502.05475}.

\bibitem[Levin \& Nave(2007)Levin and Nave]{levin2007tensor}
Levin, M. and Nave, C.~P.
\newblock Tensor renormalization group approach to two-dimensional classical lattice models.
\newblock \emph{Physical review letters}, 99\penalty0 (12):\penalty0 120601, 2007.

\bibitem[Mielke et~al.(2021)Mielke, Alyafeai, Salesky, Raffel, Dey, Gallé, Raja, Si, Lee, Sagot, and Tan]{mielke2021between}
Mielke, S.~J., Alyafeai, Z., Salesky, E., Raffel, C., Dey, M., Gallé, M., Raja, A., Si, C., Lee, W.~Y., Sagot, B., and Tan, S.
\newblock Between words and characters: {A} {B}rief {H}istory of {O}pen-{V}ocabulary {M}odeling and {T}okenization in {NLP}, 2021.
\newblock URL \url{https://arxiv.org/abs/2112.10508}.

\bibitem[Mikolov et~al.(2013{\natexlab{a}})Mikolov, Chen, Corrado, and Dean]{mikolov2013efficient}
Mikolov, T., Chen, K., Corrado, G., and Dean, J.
\newblock Efficient estimation of word representations in vector space.
\newblock \emph{arXiv preprint arXiv:1301.3781}, 2013{\natexlab{a}}.

\bibitem[Mikolov et~al.(2013{\natexlab{b}})Mikolov, Sutskever, Chen, Corrado, and Dean]{mikolov2013distributed}
Mikolov, T., Sutskever, I., Chen, K., Corrado, G.~S., and Dean, J.
\newblock Distributed representations of words and phrases and their compositionality.
\newblock \emph{Advances in neural information processing systems}, 26, 2013{\natexlab{b}}.

\bibitem[Mumford(1994)]{mumford1994pattern}
Mumford, D.
\newblock Pattern theory: a unifying perspective.
\newblock In \emph{First European Congress of Mathematics: Paris, July 6-10, 1992 Volume I Invited Lectures (Part 1)}, pp.\  187--224. Springer, 1994.

\bibitem[Navarro et~al.(2024)Navarro, Rozada, Marques, and Segarra]{navarro2024low}
Navarro, M., Rozada, S., Marques, A.~G., and Segarra, S.
\newblock Low-rank tensors for multi-dimensional markov models.
\newblock \emph{arXiv preprint arXiv:2411.02098}, 2024.

\bibitem[Ohta(2021)]{ohta2021realization}
Ohta, Y.
\newblock On the realization of hidden markov models and tensor decomposition.
\newblock \emph{IFAC-PapersOnLine}, 54\penalty0 (9):\penalty0 725--730, 2021.

\bibitem[Orioli \& Faccioli(2016)Orioli and Faccioli]{orioli2016dimensional}
Orioli, S. and Faccioli, P.
\newblock Dimensional reduction of markov state models from renormalization group theory.
\newblock \emph{The Journal of chemical physics}, 145\penalty0 (12), 2016.

\bibitem[Oseledets(2011)]{oseledetsTensorTrainDecomposition2011}
Oseledets, I.~V.
\newblock Tensor-train decomposition.
\newblock \emph{SIAM Journal on Scientific Computing}, 33\penalty0 (5):\penalty0 2295--2317, 2011.

\bibitem[Pal et~al.(2023)Pal, Sun, Yuan, Wallace, and Bau]{pal-etal-2023-future}
Pal, K., Sun, J., Yuan, A., Wallace, B., and Bau, D.
\newblock Future lens: Anticipating subsequent tokens from a single hidden state.
\newblock In Jiang, J., Reitter, D., and Deng, S. (eds.), \emph{Proceedings of the 27th Conference on Computational Natural Language Learning (CoNLL)}, pp.\  548--560, Singapore, December 2023. Association for Computational Linguistics.
\newblock \doi{10.18653/v1/2023.conll-1.37}.
\newblock URL \url{https://aclanthology.org/2023.conll-1.37/}.

\bibitem[Pennington et~al.(2014)Pennington, Socher, and Manning]{pennington2014glove}
Pennington, J., Socher, R., and Manning, C.~D.
\newblock Glove: Global vectors for word representation.
\newblock In \emph{Proceedings of the 2014 conference on empirical methods in natural language processing (EMNLP)}, pp.\  1532--1543, 2014.

\bibitem[Phuong \& Hutter(2022)Phuong and Hutter]{phuong2022formal}
Phuong, M. and Hutter, M.
\newblock Formal algorithms for transformers.
\newblock Preprint arXiv:2207.09238 [cs.LG], 2022.

\bibitem[Polchinski(1984)]{polchinski1984renormalization}
Polchinski, J.
\newblock Renormalization and effective lagrangians.
\newblock \emph{Nuclear Physics B}, 231\penalty0 (2):\penalty0 269--295, 1984.

\bibitem[Reed \& Simon(1980)Reed and Simon]{reed1980methods}
Reed, M. and Simon, B.
\newblock Methods of modern mathematical physics i: Functional analysis, acad.
\newblock \emph{Press, San Diego}, 1980.

\bibitem[Rogers \& McClelland(2004)Rogers and McClelland]{rogers2004semantic}
Rogers, T.~T. and McClelland, J.~L.
\newblock \emph{Semantic Cognition: A Parallel Distributed Processing Approach}.
\newblock MIT Press, 2004.

\bibitem[Rust et~al.(2021)Rust, Pfeiffer, Vuli{\'c}, Ruder, and Gurevych]{rust2021good}
Rust, P., Pfeiffer, J., Vuli{\'c}, I., Ruder, S., and Gurevych, I.
\newblock How good is your tokenizer? on the monolingual performance of multilingual language models.
\newblock In Zong, C., Xia, F., Li, W., and Navigli, R. (eds.), \emph{Proceedings of the 59th Annual Meeting of the Association for Computational Linguistics and the 11th International Joint Conference on Natural Language Processing (Volume 1: Long Papers)}, pp.\  3118--3135, Online, August 2021. Association for Computational Linguistics.
\newblock \doi{10.18653/v1/2021.acl-long.243}.
\newblock URL \url{https://aclanthology.org/2021.acl-long.243/}.

\bibitem[Seabrook \& Wiskott(2023)Seabrook and Wiskott]{seabrook2023tutorial}
Seabrook, E. and Wiskott, L.
\newblock A tutorial on the spectral theory of markov chains.
\newblock \emph{Neural Computation}, 35\penalty0 (11):\penalty0 1713--1796, 2023.

\bibitem[Sennrich et~al.(2016)Sennrich, Haddow, and Birch]{sennrich2016neural}
Sennrich, R., Haddow, B., and Birch, A.
\newblock Neural machine translation of rare words with subword units.
\newblock In \emph{Proceedings of the 54th Annual Meeting of the Association for Computational Linguistics (Volume 1: Long Papers)}, pp.\  1715--1725, 2016.

\bibitem[Shai et~al.(2024)Shai, Teixeira, Oldenziel, Marzen, and Riechers]{shai2024transformers}
Shai, A., Teixeira, L., Oldenziel, A., Marzen, S., and Riechers, P.
\newblock Transformers represent belief state geometry in their residual stream.
\newblock \emph{Advances in Neural Information Processing Systems}, 37:\penalty0 75012--75034, 2024.

\bibitem[Shannon(1948)]{shannon1948mathematical}
Shannon, C.~E.
\newblock A mathematical theory of communication.
\newblock \emph{The Bell system technical journal}, 27\penalty0 (3):\penalty0 379--423, 1948.

\bibitem[Thibeault et~al.(2024)Thibeault, Allard, and Desrosiers]{thibeault2024low}
Thibeault, V., Allard, A., and Desrosiers, P.
\newblock The low-rank hypothesis of complex systems.
\newblock \emph{Nature Physics}, 20\penalty0 (2):\penalty0 294--302, 2024.

\bibitem[Tucker(1966)]{tucker1966some}
Tucker, L.~R.
\newblock Some mathematical notes on three-mode factor analysis.
\newblock \emph{Psychometrika}, 31\penalty0 (3):\penalty0 279--311, 1966.

\bibitem[Urdshals \& Urdshals(2025)Urdshals and Urdshals]{urdshals2025}
Urdshals, E. and Urdshals, J.
\newblock Structure development in list-sorting transformers, 2025.
\newblock URL \url{https://arxiv.org/abs/2501.18666}.

\bibitem[Vaart(1998)]{Vaart_1998}
Vaart, A. W. v.~d.
\newblock \emph{Asymptotic Statistics}.
\newblock Cambridge Series in Statistical and Probabilistic Mathematics. Cambridge University Press, 1998.

\bibitem[Vaintrob(2025)]{vaintrob}
Vaintrob, D.
\newblock The generalization phase diagram.
\newblock \url{https://www.lesswrong.com/posts/TSe3qhe4kxgKPJmfD/the-generalization-phase-diagram}, 2025.

\bibitem[Vaswani et~al.(2017)Vaswani, Shazeer, Parmar, Uszkoreit, Jones, Gomez, Kaiser, and Polosukhin]{vaswani2017attention}
Vaswani, A., Shazeer, N., Parmar, N., Uszkoreit, J., Jones, L., Gomez, A.~N., Kaiser, {\L}., and Polosukhin, I.
\newblock Attention is all you need.
\newblock \emph{Advances in neural information processing systems}, 30, 2017.

\bibitem[Wang et~al.(2024)Wang, Hoogland, van Wingerden, Furman, and Murfet]{wang2024differentiationspecializationattentionheads}
Wang, G., Hoogland, J., van Wingerden, S., Furman, Z., and Murfet, D.
\newblock Differentiation and specialization of attention heads via the refined local learning coefficient, 2024.
\newblock URL \url{https://arxiv.org/abs/2410.02984}.

\bibitem[Watanabe(2009)]{watanabeAlgebraicGeometryStatistical2009}
Watanabe, S.
\newblock \emph{Algebraic {{Geometry}} and {{Statistical Learning Theory}}}.
\newblock {Cambridge University Press}, {USA}, 2009.

\bibitem[Watanabe(2013)]{watanabeWidelyApplicableBayesian2013}
Watanabe, S.
\newblock A {{Widely Applicable Bayesian Information Criterion}}.
\newblock \emph{Journal of Machine Learning Research}, 14:\penalty0 867--897, 2013.

\bibitem[Watanabe(2018)]{watanabe2018}
Watanabe, S.
\newblock \emph{Mathematical {Theory} of {Bayesian} {Statistics}}.
\newblock {CRC Press, Taylor and Francis group}, {USA}, 2018.

\bibitem[Welling \& Teh(2011)Welling and Teh]{wellingBayesianLearningStochastic2011}
Welling, M. and Teh, Y.~W.
\newblock Bayesian {L}earning via {S}tochastic {G}radient {L}angevin {D}ynamics.
\newblock In \emph{Proceedings of the 28th {{International Conference}} on {{Machine Learning}}}, 2011.

\bibitem[Wu et~al.(2024)Wu, Morris, and Levine]{wu2024do}
Wu, W., Morris, J.~X., and Levine, L.
\newblock Do language models plan ahead for future tokens?
\newblock In \emph{First Conference on Language Modeling}, 2024.
\newblock URL \url{https://openreview.net/forum?id=BaOAvPUyBO}.

\bibitem[Yao et~al.(2020)Yao, Gholami, Keutzer, and Mahoney]{yao2020pyhessianneuralnetworkslens}
Yao, Z., Gholami, A., Keutzer, K., and Mahoney, M.
\newblock Pyhessian: Neural networks through the lens of the hessian, 2020.
\newblock URL \url{https://arxiv.org/abs/1912.07145}.

\bibitem[Zhang \& Wang(2019)Zhang and Wang]{zhang2019spectral}
Zhang, A. and Wang, M.
\newblock Spectral state compression of markov processes.
\newblock \emph{IEEE transactions on information theory}, 66\penalty0 (5):\penalty0 3202--3231, 2019.

\end{thebibliography}
\bibliographystyle{icml2024}

\appendix

\section{Multiple Sequence Lengths}\label{section:mult_seq_lengths}

For any $1 \leq k \leq K$, define
\[
    L_k(w) = -\sum_{X_1 \cdots X_k} q(X_1 \cdots X_k) \log p(X_1 \cdots X_k, w)\,.
\]
For $k$, $l$ with $k + l \le K$ define 
\begin{align}
    L_{k, l}(w) 
    &= -\sum_{X_1 \cdots X_k X_1 \cdots X_l} q(X_1 \cdots X_k X_1 \cdots X_l) \log p(X_1 \cdots X_l | X_1 \cdots X_k, w) \,.
\end{align}
A \emph{decomposition of $K$ of length $m$} is defined as  a sequence of pairs $(k_1, l_1), (k_2, l_2), \dots, (k_m, l_m)$ such that
\begin{itemize}
    \item  $k_i, l_i \in \mathbb{Z}_{>0}$ for $i = 1, \dots, m$;
    \item $k_1 + l_1 = K$;
    \item $k_i = k_{i+1} + l_{i+1}$ for $i = 1, \dots, m-1$.
\end{itemize}
Given a decomposition $(k_1, l_1), (k_2, l_2), \dots, (k_m, l_m)$ of $K$, we have
\begin{align}
    L(w) &= -\sum_{X_1 \cdots X_K} q(X_1 \cdots X_K) \log p(X_1 \cdots X_K, w)  \nonumber \\
    &= L_{k_1, l_1}(w) + L_{k_2, l_1}(w) + \cdots + L_{k_m, l_m}(w) + L_{k_m}(w) \,.
\end{align}
For each $i = 1, \dots, m$, let $\chi_i$ be a mode for $(k_i, l_i)$. Let $q^{(\chi_i)}(y|x)$ be the corresponding truncated distribution for $x \in \Sigma^{k_i}$ and $y \in \Sigma^{l_i}$. Consider the distribution
\[
    q^{(\bar\chi)} = q^{(\chi_1)}(X_1 \cdots X_{l_1}|X_1 \cdots X_{k_1}) \cdots q^{(\chi_m)}(X_1 \cdots X_{l_m}|X_1 \cdots X_{k_m}) q(X_1 \cdots X_{k_m}) \,.
\]
Define
\[
    L^{(\bar\chi)}(w) = - \sum_{X_1 \cdots X_K} q^{(\bar\chi)}(X_1 \cdots X_K) \log p(X_1 \cdots X_K, w).
\]
Then 
\[
    L^{(\bar\chi)}(w) = L^{(\bar\chi)}_{k_1, l_1}(w) +  L^{(\bar\chi)}_{k_2, l_2}(w) + \cdots +  L^{(\bar\chi)}_{k_m, l_m}(w) + L_{k_m}(w),
\]
where for each $i =1, \dots, m$,
\[
    L^{(\bar\chi)}_{k_i, l_i}(w) = \sum_{X_1 \cdots X_{k_i}X_1 \cdots X_{l_{i}}} q^{(\chi_i)}(X_1 \cdots X_{k_i }X_1 \cdots X_{l_i}) \log p(X_1 \cdots X_{l_i} | X_1 \cdots X_{k_i}) \,.
\]
Thus 
\begin{align*}
    \big| L(w) - L^{(\bar\chi)}(w) \big| \le \sum_{i=1}^m \big| L_{k_i, l_i}(w) - L^{(\bar\chi)}_{k_i, l_i}(w) \big| \,,
\end{align*}
and 
\begin{align*}
    \big\| \nabla_w L(w) - \nabla_w L^{(\bar\chi)}(w) \big\|_2 \le \sum_{i=1}^m \big\| \nabla_w L_{k_i, l_i}(w) - \nabla_w L^{(\bar\chi)}_{k_i, l_i}(w) \big\|_2 \,.
\end{align*}
If $p(y|x, w)$, for $y \in \Sigma^{l_i}$ and $x \in \Sigma^{k_i}$, is log-probability-insensitive (\emph{resp.} gradient-insensitive) for the constant $\epsilon_i$, then we have
\begin{align*}
    \big| L(w) - L^{(\bar\chi)}(w) \big| \le \sum_{i=1}^m \big| L_{k_i, l_i}(w) - L^{(\bar\chi)}_{k_i, l_i}(w) \big| < \sum_{i=1}^m \epsilon_i \,,
\end{align*}
and
\begin{align*}
    \big\| \nabla_w L(w) - \nabla_w L^{(\bar\chi)}(w) \big\|_2 \le \sum_{i=1}^m \big\| \nabla_w L_{k_i, l_i}(w) - \nabla_w L^{(\bar\chi)}_{k_i, l_i}(w) \big\|_2 < \sum_{i=1}^m \epsilon_i \,,
\end{align*}
respectively.

\section{Methodology for Empirical Modes}\label{section:token_svd}

This appendix describes the methodology used to analyse SVD of conditional probability matrices.

We used the Pile dataset \citep{gao2020pile}. For tokenisation, we employed the EleutherAI/pythia-70m tokenizer, which has a vocabulary size of approximately 50,000 tokens. For a given value of preceding sequence length $k$ and following sequence length $l$, we processed text documents to extract and count token sequences:

\begin{enumerate}
    \item We sampled $N$ documents from the dataset, where $N$ varies based on sequence length complexity:
    \begin{itemize}
        \item For $k \leq 1, l \leq 1$: $N = 20,000$ documents
        \item For $k \leq 2, l \leq 2$: $N = 15,000$ documents
        \item For $k \geq 3, l \geq 3$: $N = 25,000$ documents
    \end{itemize}
    
    \item For each document, we tokenised the text and extracted:
    \begin{itemize}
        \item $X$ sequences: consecutive token sequences of length $k$
        \item $Y$ sequences: consecutive token sequences of length $l$ that immediately follow an X sequence
        \item $XY$ sequences: concatenated X and Y sequences representing co-occurrences
    \end{itemize}
    
    \item We filtered $X$ sequences to include only those occurring at least $min\_count$ times:
    \begin{itemize}
        \item For $k \leq 1, l \leq 1$: $min\_count = 5$
        \item For $k \leq 2, l \leq 2$: $min\_count = 10$
        \item For $k \geq 3, l \geq 3$: $min\_count = 20$
    \end{itemize}
    
    \item Similarly, we filtered $Y$ sequences to include only those occurring at least $min\_y\_count$ times following a filtered $X$ sequence, using the same thresholds as above.
\end{enumerate}
Next we construct a sparse matrix $M$ where:
\begin{itemize}
    \item Rows correspond to $Y$ sequences ($i$ indexes $Y$ sequences)
    \item Columns correspond to $X$ sequences ($j$ indexes $X$ sequences)
    \item Each cell $M_{i,j}$ represents the conditional probability $P(Y_i | X_j)$
\end{itemize}

The conditional probability with Laplace smoothing is calculated as:
\begin{equation}
P(Y_i | X_j) = \frac{count(X_j, Y_i) + \lambda}{count(X_j) + \lambda \cdot |Y|}
\end{equation}

where:
\begin{itemize}
    \item $count(X_j, Y_i)$ is the number of times sequence $Y_i$ follows sequence $X_j$
    \item $count(X_j)$ is the total number of occurrences of sequence $X_j$
    \item $|Y|$ is the total number of distinct $Y$ sequences after filtering
    \item $\lambda = 10^{-5}$ is the smoothing parameter
\end{itemize}
We applied SVD to the conditional probability matrix
\begin{equation}
M = U \Sigma V^T
\end{equation}
where:
\begin{itemize}
    \item $U$ is the matrix of left singular vectors (size $|Y| \times r$)
    \item $\Sigma$ is the diagonal matrix of singular values (size $r \times r$)
    \item $V^T$ is the transpose of the matrix of right singular vectors (size $r \times |X|$)
    \item $r$ = 100 is the number of components we compute
\end{itemize}
For large matrices (min dimension > 5000), we used randomized SVD from scikit-learn's TruncatedSVD for computational efficiency. For smaller matrices, we used scipy's direct svds method.

To illustrate the patterns identified by SVD, we extracted contextual examples from the corpus that showcase the token relationships captured by the top singular components:

\begin{enumerate}
    \item For each singular component $i$:
    \begin{itemize}
        \item We identified the $Y$ sequence with the highest magnitude loading in the left singular vector $U_{\cdot,i}$
        \item We identified $X$ sequences with high magnitude loadings in the right singular vector $V^T_{i,\cdot}$
    \end{itemize}
    
    \item We searched the corpus for occurrences where:
    \begin{itemize}
        \item The $Y$ sequence exactly matches the top $Y$ sequence for the component
        \item The $X$ sequence matches one of the high-loading $X$ sequences
    \end{itemize}
    
    \item We used an adaptive threshold approach, starting with $X$ sequences within 10\% of the maximum loading magnitude and gradually relaxing to 50\% if needed to find examples.
    
    \item For each match, we extracted context surrounding the sequences (50 tokens before and after) to provide a readable example.
\end{enumerate}

This approach allowed us to identify and visualise the contextual patterns captured by each SVD component.

\section{More Examples of Modes}\label{appendix:more_modes}

In Figure \ref{fig:k3_l3_example1} and Figure \ref{fig:k3_l3_example2} we give some additional examples of empirical modes in the Pile.

\begin{figure}[tbp]
    \centering
    \includegraphics[width=\textwidth]{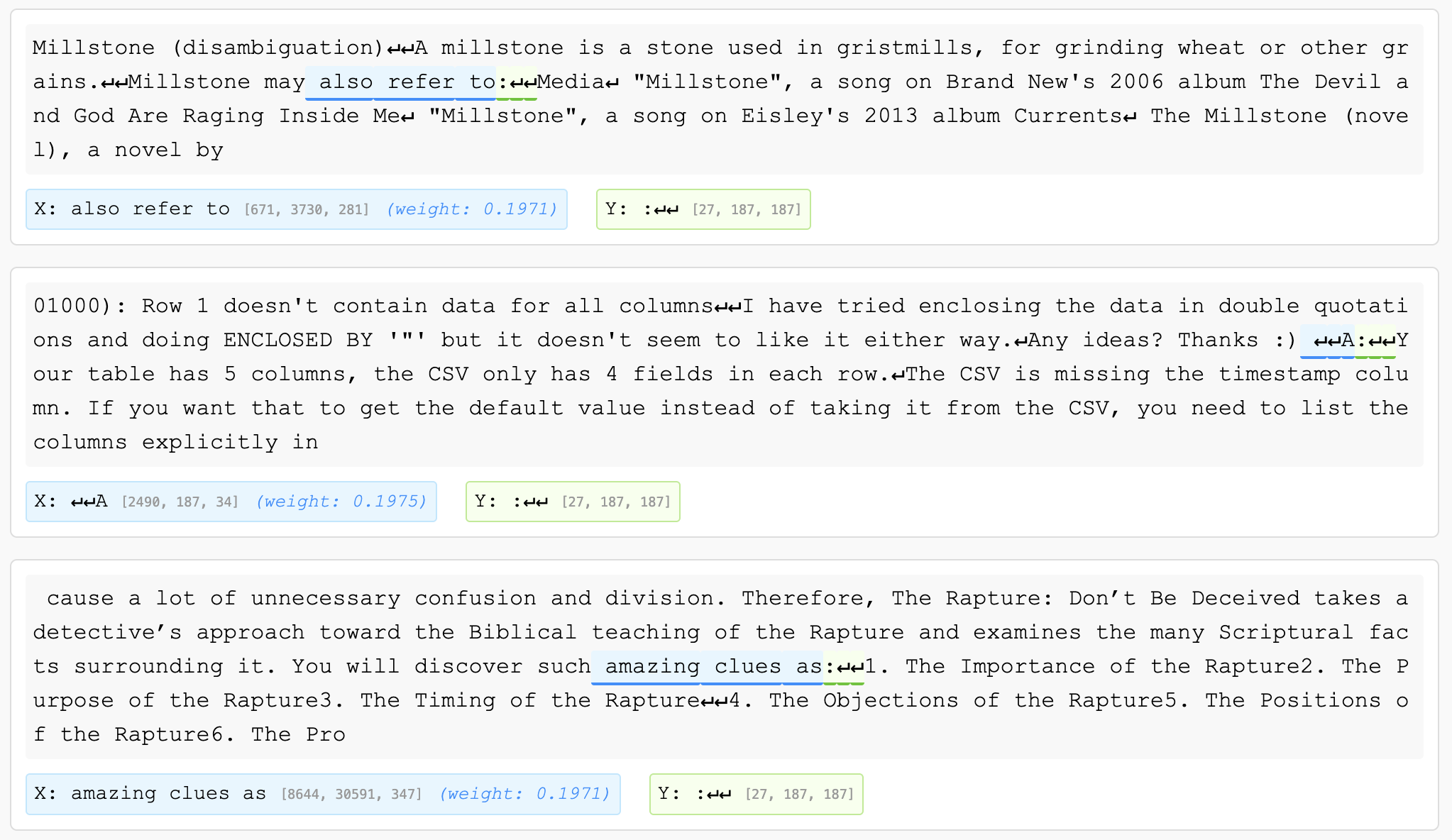}
    \caption{\textbf{Empirical modes.} We show an example of the $x,y$ pair which are heavily loaded in the first empirical mode for $k = 3$, $l = 3$ in our experiments on the Pile. Shown are three text samples. Next to each $X, Y$ token sequence we show the indices in the tokeniser.}
    \label{fig:k3_l3_example1}
\end{figure}

\begin{figure}[tbp]
    \centering
    \includegraphics[width=\textwidth]{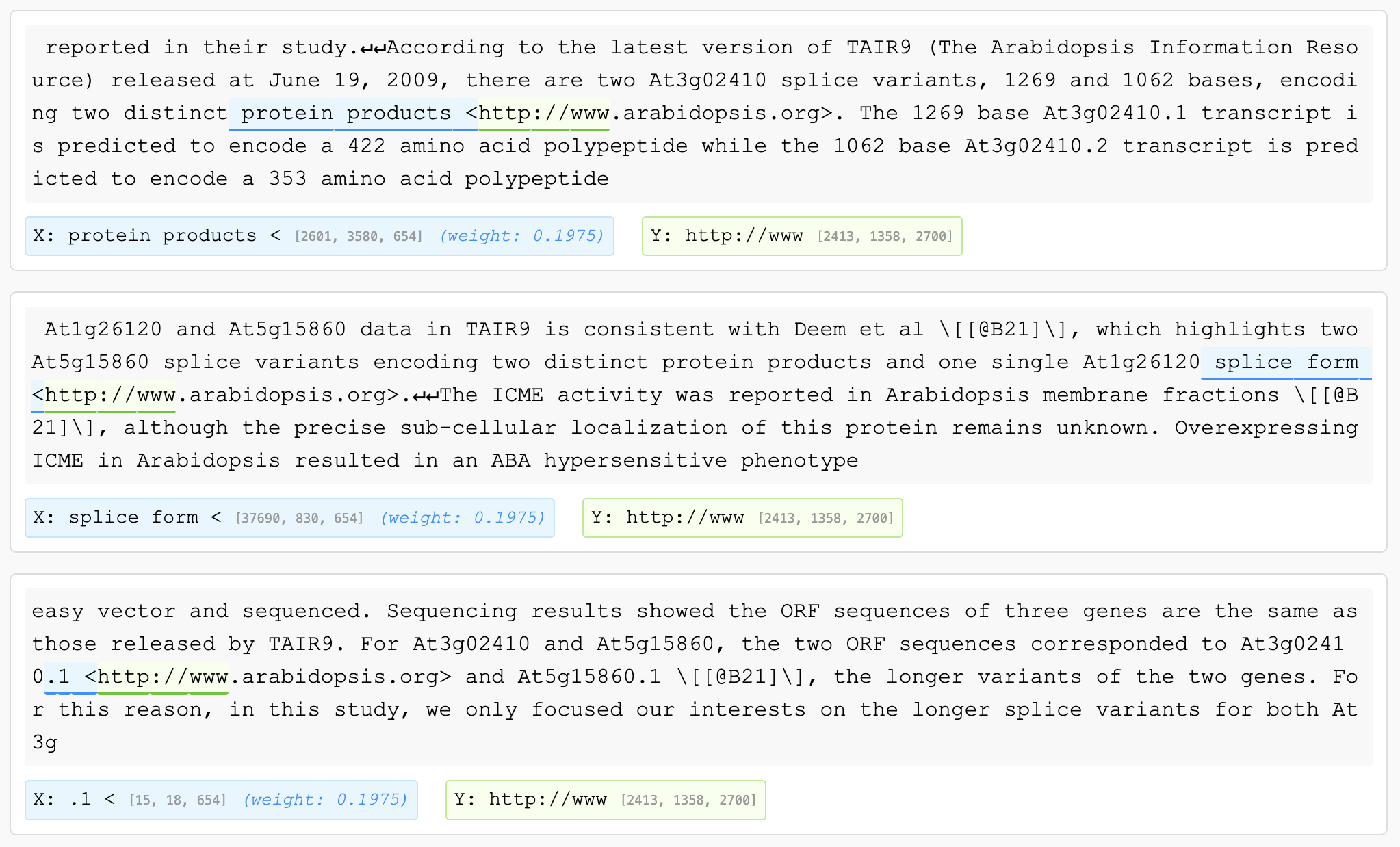}
    \caption{\textbf{Empirical modes.} We show an example of the $x,y$ pair which are heavily loaded in the first empirical mode for $k = 3$, $l = 3$ in our experiments on the Pile. Shown are three text samples. Next to each $X, Y$ token sequence we show the indices in the tokeniser.}
    \label{fig:k3_l3_example2}
\end{figure}

\section{Higher-Order Tensors}\label{section:hod_tensor}

In the main text, we used SVD to provide a factorization of the fundamental tensor. For $k > 2$, there is no canonical factorization of $A_k$ and the choice of factorisation becomes an additional consideration. Natural candidates include the CANDECOMP/PARAFAC (CP) decomposition \cite{harshman1970foundations, carroll1970analysis}, the Tucker decomposition \cite{tucker1966some}, and matrix-product states (MPS) or tensor-trains \cite{oseledetsTensorTrainDecomposition2011}. For more background see \citep{kolda2009tensor}.

Even once a particular kind of decomposition is chosen, there are still additional choices that can be made, which affect the ultimate basis of function space that we end up with. We focus on the Tucker decomposition of $A_k \in (\mathbb{R}^\Sigma)^{\otimes k}$. Set $V = \mathbb{R}^\Sigma$ and write $V_i$ for labelled copies of $V$ so
\[
(\mathbb{R}^\Sigma)^{\otimes k} = V_1 \otimes \cdots \otimes V_k\,.
\]
We now allow for rearrangement and grouping of these tensor factors. More precisely we assume chosen a partition $P_1, \ldots, P_r$ of $\{ 1, \ldots, k \}$ and set $V'_j = \bigotimes_{i \in P_j} V_i$ so there is a canonical isomorphism
\begin{equation}\label{eq:grouping_tucker}
(\mathbb{R}^\Sigma)^{\otimes k} \cong V'_1 \otimes \cdots \otimes V'_r\,.
\end{equation}
Under this isomorphism we view $A_k$ as a tensor in $\bigotimes_j V'_j$. Using the word basis to identify vector spaces with their duals, we consider for each $1 \le j \le r$ the isomorphism
\[
V'_1 \otimes \cdots \otimes V'_r \cong \Big[ \bigotimes_{k \neq j} V'_k \Big]^* \otimes V'_j\,.
\]
As an element of the right hand side $A_k$ determines a linear transformation with codomain $V'_j$ and we let $\{ v^{(j)}_\alpha \}_{\alpha \in \Lambda_j}$ be an orthonormal basis for $V'_j$ consisting of the left singular vectors of an SVD decomposition of this transformation. The tensor products of these bases over $j$ give an orthonormal basis of $V'_1 \otimes \cdots \otimes V'_r$ with respect to which we can write, as elements of $\bigotimes_j V'_j$,
\begin{equation}\label{eq:tucker_decomp_ak}
A_k = \sum_{\alpha \in \Lambda} \lambda_{\alpha} v^{(1)}_{\alpha_1} \otimes \cdots \otimes v^{(r)}_{\alpha_r} 
\end{equation}
where $\Lambda = \Lambda_1 \times \cdots \times \Lambda_r$ and $\lambda_{\alpha} \in \mathbb{R}$. This is called the \emph{Tucker decomposition} with respect to the grouping \eqref{eq:grouping_tucker}. Usually the Tucker decomposition is presented using the \emph{core tensor}
\begin{equation}\label{eq:tucker_decomp_truncate}
G = \sum_{\alpha \in \Lambda} \lambda_\alpha \alpha_1 \otimes \cdots \otimes \alpha_r \in \mathbb{R}^{\Lambda_1} \otimes \cdots \otimes \mathbb{R}^{\Lambda_r} \cong \mathbb{R}^\Lambda
\end{equation}
which can be contracted against tensors $U^{(j)} = \sum_{\alpha_j} \alpha_j^* \otimes v^{(j)}_{\alpha_j}$ to produce $A_k$. The idea is that if the core tensor $G$ is much lower rank than $A_k$, then the Tucker decomposition has succeeded in representing that tensor in a compressed form; moreover, since the sets $\Lambda_j$ can be ordered by the corresponding singular values, by truncating the sum in \eqref{eq:tucker_decomp_truncate} we obtain a natural way of discarding ``higher order'' information in $A_k$.

These more general tensor decompositions provide a ``drop in'' replacement for the orthonormal bases produced by SVD in the main text. Additionally, we can use \emph{multiple} different tensor decompositions for $A_k$ and use the union of the orthonormal bases thus obtained; in this way we get an overcomplete basis rather than a standard basis. Ultimately this may be the more appropriate way of thinking about modes in models like transformers, which are likely to learn modes for multiple distinct factorisations of the fundamental tensor.

\end{document}